%% file: main.tex
\documentclass{article}

\usepackage{arxiv}

\input{math_commands.tex}

\usepackage[utf8]{inputenc} % allow utf-8 input
\usepackage[T1]{fontenc}    % use 8-bit T1 fonts
\usepackage{hyperref}       % hyperlinks
\usepackage{url}            % simple URL typesetting
\usepackage{booktabs}       % professional-quality tables
\usepackage{amsfonts}       % blackboard math symbols
\usepackage{nicefrac}       % compact symbols for 1/2, etc.
\usepackage{microtype}      % microtypography
\usepackage{lipsum}		% Can be removed after putting your text content
\usepackage{graphicx}
\usepackage{natbib}
\usepackage{doi}

\usepackage{hyperref}
\usepackage{url}

\usepackage{amsmath}
\usepackage{amsfonts}
\usepackage{mathtools}
\usepackage{amsthm}
\usepackage{thmtools}
\usepackage{thm-restate}
\usepackage{delimset}
\usepackage{xparse}
\usepackage{xcolor}
\usepackage{cleveref}
\usepackage{algorithm}
\usepackage{algorithmic}
\usepackage{tabularx}
\usepackage{booktabs}
\usepackage{subfigure}
\usepackage{tikz}
\usepackage{siunitx}
\newcommand{\A}{\mathcal{A}}
\newcommand{\X}{\mathcal{X}}
\newcommand{\M}{\mathcal{M}}
\newcommand{\G}{\mathcal{G}}
\newcommand{\F}{\mathcal{F}}
\newcommand{\D}{\mathcal{D}}

\newcommand{\s}{\mathcal{S}}
\newcommand{\Z}{\mathcal{Z}}

\renewcommand\abs[1]{\left|#1\right|}

\newcommand{\Mod}[1]{\ (\mathrm{mod}\ #1)}

\newcommand{\equivset}[1]{\Upsilon_{#1}}

\newcommand{\Ponline}{\rho_{o}}
\newcommand{\Poffline}{\rho_{e}}
\newcommand{\otherPoffline}{\widetilde{\rho}_{e}}
\newcommand{\donline}{d_{\Ponline}}
\newcommand{\doffline}{d_{\Poffline}}

\DeclarePairedDelimiterX{\infdivx}[2]{(}{)}{%
  #1\;\delimsize|\delimsize|\;#2%
}

\DeclarePairedDelimiterXPP\expect[2]{\mathbb{E}_{#1}}[]{}{\setargs{#2}}%
\NewDocumentCommand{\setargs}{>{\SplitArgument{1}{|}}m}
{\setargsaux#1}
\NewDocumentCommand{\setargsaux}{mm}
{\IfNoValueTF{#2}{#1}{\nonscript\,#1\nonscript\;\delimsize\vert\nonscript\:\allowbreak #2\nonscript\,}}
\DeclarePairedDelimiterXPP\expectaux[3]{\mathbb{E}_{#1}}[]{}{#2\nonscript\:\delimsize\vert\nonscript\:#3}%

\newtheorem*{theorem*}{Theorem}
\newtheorem{lemma}{Lemma}

\newtheorem{definition}{Definition}
\newtheorem{assumption}{Assumption}
\newtheorem{remark}{Remark}
\newtheorem*{remark*}{Remark}

\title{On Covariate Shift of Latent Confounders \\ in Imitation and Reinforcement Learning}

\author{ Guy Tennenholtz \thanks{Correspondence to 
\texttt{guytenn@gmail.com}} \\
	Technion University \& \\
	Nvidia Research
	\And
	Assaf Hallak \\
	Nvidia Research
	\And
	Gal Dalal \\
	Nvidia Research
	\AND
	Shie Mannor \\
	Technion University \& \\
	Nvidia Research
	\And 
	Gal Chechik \\
	Nvidia Research
    \And
    Uri Shalit \\
    Technion University
	
}

\renewcommand{\shorttitle}{On Covariate Shift of Latent Confounders}

\hypersetup{
pdftitle={A template for the arxiv style},
pdfsubject={q-bio.NC, q-bio.QM},
pdfauthor={David S.~Hippocampus, Elias D.~Striatum},
pdfkeywords={First keyword, Second keyword, More},
}

\begin{document}
\maketitle

\begin{abstract}
	We consider the problem of using expert data with unobserved confounders for imitation and reinforcement learning. We begin by defining the problem of learning from confounded expert data in a contextual MDP setup. We analyze the limitations of learning from such data with and without external reward, and propose an adjustment of standard imitation learning algorithms to fit this setup. %In addition, 
    We then discuss the problem of distribution shift between the expert data and the online environment when the data is only partially observable. We prove possibility and impossibility results for imitation learning under arbitrary distribution shift of the missing covariates. When additional external reward is provided, we propose a sampling procedure that addresses the unknown shift and prove convergence to an optimal solution. Finally, we validate our claims empirically on challenging assistive healthcare and recommender system simulation tasks.
\end{abstract}

\section{Introduction}

Reinforcement Learning (RL) is increasingly used across many fields to create agents that learn via interaction and reward feedback. % \citep{vinyals2019grandmaster,tessler2021reinforcement,mandel2014offline}.
In many such cases, we rely on experts to perform certain tasks, integrating their knowledge to improve learning efficiency and overall performance. Imitation Learning~(IL, \citet{hussein2017imitation}) is concerned with learning via expert demonstrations without access to a reward function. Similarly, RL settings often use expert data to boost performance, eliminating the need to learn from scratch. In this work we consider the IL and RL paradigms in the presence of partially observable expert data. 

While expert demonstration data is useful, in many realistic settings such data may be prone to hidden confounding \citep{gottesman2019guidelines}, i.e., there may be features used by the expert which are not observed by the learning agent. This can occur due to, e.g., privacy constraints, continually changing features in ongoing production pipelines, or when not all information available to the human expert was recorded. As we show in our work, covariate shift of unobserved factors between the expert data and the real world may lead to significant negative impact on performance, frequently rendering the data useless for imitation (see \Cref{fig: assistive illustration} and \Cref{thm: impossibility}). 

In this paper we define the tasks of imitation and reinforcement learning using expert data with unobserved confounders and possible covariate shift. We focus on a contextual MDP setting \citep{hallak2015contextual}, where a \emph{context} is sampled at every episode from some distribution, affecting both the reward and the transition between states. We assume that the agent has access to additional expert data, generated by an optimal policy, for which the sampled context is missing, yet \emph{is observed} in the online environment. 
% In the assistive healthcare setting, this may occur when the expert does not provide information on specific preferences of the patient, yet this information is available (for new patients) when interacting in the real-world.

We begin by analyzing the imitation-learning problem, (i.e., without access to reward) in \Cref{section: imitation}. Under no covariate shift in the unobserved context, we characterize a sufficient and necessary set of optimal policies. In contrast, we prove that in the presence of a covariate shift, if the true reward depends on the context, then the imitation-learning problem is non-identifiable and prone to catastrophic errors (see \Cref{section: covariate shift} and \Cref{thm: impossibility}). We further analyze the RL setting (i.e., with access to reward and confounded expert data) in \Cref{section: rl}. \Cref{fig: assistive illustration} depicts a possible failure case of using confounded expert data with unknown covariate shift in a dressing task. Unlike the imitation setting, we show that optimality can still be achieved in the RL setting while using confounded expert data with arbitrary covariate shift. We use a corrective data sampling procedure and prove convergence to an optimal policy. 

Our contributions are as follows. (1) We introduce IL and RL with hidden confounding and prove fundamental characteristics w.r.t. covariate shift and the feasibility of imitation. (2) In the RL setting, under arbitrary covariate shift, we provide a novel algorithm with convergence guarantees which uses a corrective sampling technique to account for the unknown context distribution in the expert data. (3) Finally, we conduct extensive experiments on recommender system \citep{ie2019recsim} and assistive-healthcare \citep{erickson2020assistive} environments, demonstrating our theoretical results, and suggesting that confounded expert data can be used in a controlled manner to improve the efficiency and performance of RL agents.

\begin{figure}[t!]
\centering
\includegraphics[width=0.99\linewidth]{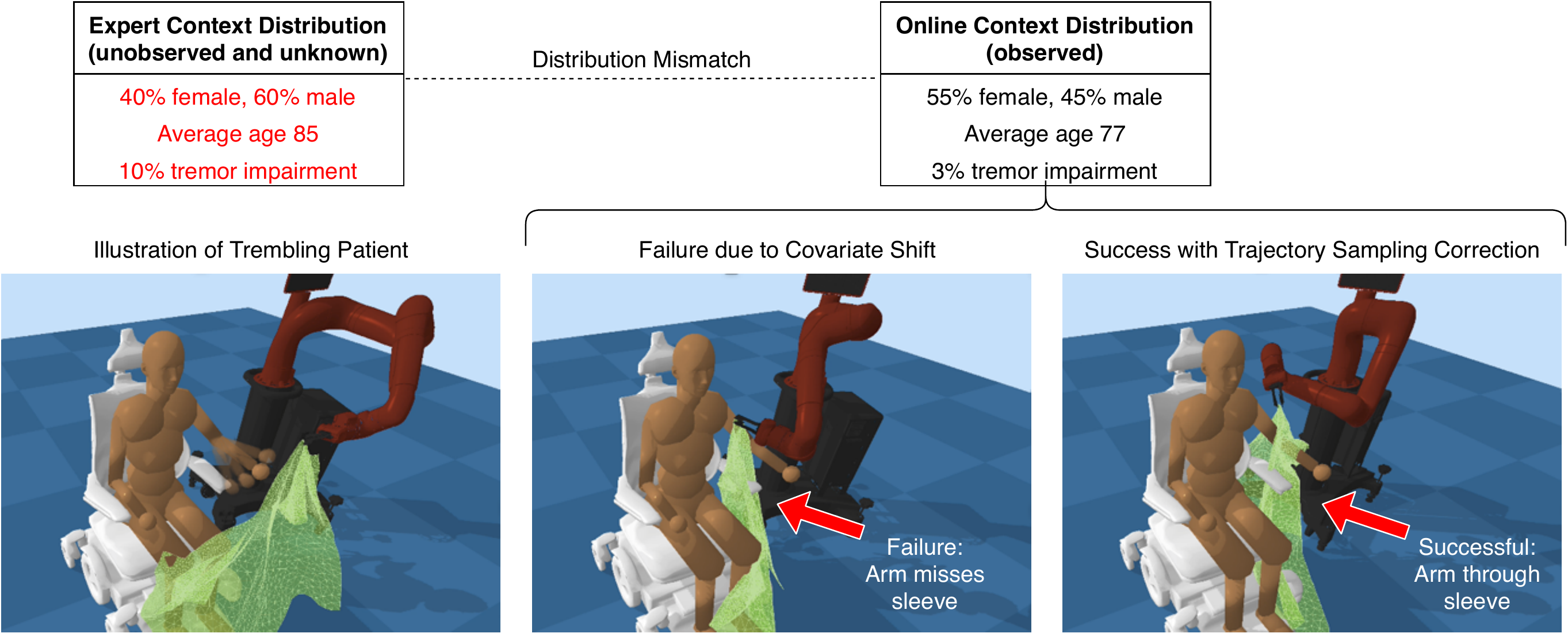}
\caption{\small Failure of using confounded expert data under context distribution mismatch between online environment and expert data. Caregiver does not learn to perform well in a dressing task when covariate shift of hidden confounders is present but not accounted for.}
\label{fig: assistive illustration}
\end{figure}

% \vspace{-0.2cm}
\section{Preliminaries}
\label{section: preliminaries}
% \vspace{-0.05cm}

\paragraph{Online Environment.}
We consider a contextual MDP \citep{hallak2015contextual} defined by the tuple ${\M = (\s, \X, \A, P, r, \Ponline, \nu, \gamma)}$, where $\s$ is the state space, $\X$ is the context space, $\A$ is the action space, $P: \s \times \s \times \A \times X \mapsto [0, 1]$ is the context dependent transition kernel, $r: \s \times \A \times X \mapsto [0, 1]$ is the context dependent reward function, and $\gamma \in (0, 1)$ is the discount factor. We assume an initial distribution over contexts $\Ponline: \X \mapsto [0,1]$ and an initial state distribution $\nu: \s \times \X \mapsto [0,1]$.

The environment initializes at some context  $x \sim \Ponline(\cdot)$, and state $s_0 \sim \nu(\cdot | x)$. At time $t$, the environment is at state $s_t \in \s$ and an agent selects an action $a_t \in \A$. The agent receives a reward $r_t = r(s_t, a_t, x)$ and the environment then transitions to state $s_{t+1} \sim P(\cdot | s_t, a_t, x)$.

We define a Markovian stationary policy $\pi$ as a mapping $\pi: \s \times \X \times \A \mapsto [0,1]$, such that $\pi(\cdot | s, x)$ is the action sampling probability. We define the value of a policy $\pi$ by
${
    v_\M(\pi) = \expect*{\pi}{(1-\gamma)\sum_{t=0}^\infty \gamma^t r(s_t, a_t, x) | x \sim \Ponline, s_0 \sim \nu(\cdot \mid x)},}
$
where $\mathbb{E}_\pi$ denotes the expectation induced by the policy $\pi$. We denote by $\Pi$ the set of all Markovian policies and $\Pi_{\text{det}}$ the set of deterministic Markovian policies. We define the optimal value and policy by $v^*_\M = \max_{\pi \in \Pi} v_\M(\pi)$, and ${\pi^*_{\M} \in \arg\max_{\pi \in \Pi} v_\M(\pi)}$, respectively. Whenever appropriate, we simplify notation and write $v^*, \pi^*$. We use $\Pi^*_\M$ to denote the set of optimal policies in $\M$, i.e., $\Pi^*_\M = \arg\max_{\pi \in \Pi} v_\M(\pi)$. We also define the set of catastrophic policies $\Pi^\dagger_\M$ as the set
\begin{align}
    \Pi^\dagger_\M = \arg\min_{\pi \in \Pi} v_\M(\pi).
\label{eq: catastrophic policy}
\end{align}
We later use this set to show impossibility of imitation under arbitrary covariate shift and a context-independent transition function.

\paragraph{Expert Data with Unobserved Confounders.} We assume additional access to a confounded dataset consisting of expert trajectories $\D^* = \brk[c]*{(s_0^i, a_0^i, s_1^i, a_1^i, \hdots, s_H^i, a_H^i)}_{i=1}^n$, where $a^i_j \sim \pi^* \in \Pi^*_\M$. The trajectories in the dataset were sampled i.i.d. from the marginalized expert distribution (under possible context covariate shift)
${
    P^*(s_0, a_0, s_1, a_1, \hdots, s_H)  
    = 
    \sum_x
    \Poffline(x)
    \nu(s_0|x)
    \prod_{t=0}^{H-1} 
    P(s_{t+1} | s_t, a_t, x)
    \pi^*(a_t | s_t, x),}
$
where $\Poffline$ is some distribution over contexts. Importantly, $\Poffline$ does not necessarily equal $\Ponline$ -- the distribution of contexts in the online environment. Notice that it is assumed that $\pi^*$ that generated the data had access to the context $x^i$ (i.e., $\pi^*$ is context-dependent), though it is missing in the data.

In this work, we consider two settings:
\begin{enumerate}
    \item \textbf{Confounded Imitation Learning} (\Cref{section: imitation}): The agent has access to confounded expert data (with context distribution $\Poffline$) as well as real environment $(\s, \X, \A, P, \Ponline, \nu, \gamma)$, \emph{without} access to reward.
    \item \textbf{Reinforcement Learning with Confounded Expert Data} (\Cref{section: rl}): The agent has access to confounded expert data (with context distribution $\Poffline$) as well as real environment ${\M = (\s, \X, \A, P, r, \Ponline, \nu, \gamma)}$, \emph{with} access to reward.
\end{enumerate}
In both settings we aim to find a context-dependent policy which maximizes the cumulative reward. The confounding factor here is w.r.t. the unobserved context and distribution $\Poffline$ in the offline data.

\paragraph{Marginalized Stationary Distribution. }

We denote the stationary distribution of a policy $\pi \in \Pi$ given context $x \in \X$ by $d^\pi(s,a | x) = (1-\gamma) \sum_{t=0}^\infty \gamma^t P^\pi(s_t = s, a_t = a | x, s_0 \sim \nu(\cdot | x))$,
where $P^\pi$ denotes the probability measure induced by $\pi$. Similarly, given a distribution over contexts, we define the marginalized stationary distribution of a policy $\pi$ under the corresponding context distribution by
\begin{align*}
    \donline^\pi(s,a) &= \expect*{x \sim \Ponline}{d^\pi(s,a\mid x)} \quad ~~~\text{(online environment)},\\
    \doffline^{\pi^*}(s,a) &= \expect*{x \sim \Poffline}{d^{\pi^*}(s,a\mid x)} \quad \text{(offline expert data)}.
\end{align*}

\paragraph{A Causal Perspective.} Our work sits at an intersection between the fields of RL and Causal Inference (CI). We believe it is essential to bridge the gap between these two fields, and include an interpretation of our model using CI terminology in \Cref{appendix: relation to causal inference}, where we equivalently define our objective as an intervention over the unknown distribution $\rho_e$ in a specified Structural Causal Model.

\section{Imitation Learning with Unobserved Confounders}
\label{section: imitation}

In this section, we analyze the problem of confounded imitation learning, namely, learning from expert trajectories with hidden confounders and without reward. Similar to previous work, we consider the task of imitation learning from expert data in the setting where the agent is allowed to interact with the environment \citep{ho2016generative,fu2017learning,kostrikov2019imitation,brantley2019disagreement}. In the first part of this section we assume no covariate shift between the online environment and the data is present, i.e., $\Poffline = \Ponline$. We lift this assumption in the second part, where we focus on the covariate shift of the hidden confounders.

% In the first part of this section, we assume no covariate shift between the online environment and the data is present, i.e., $\Poffline = \Ponline$. In the second part of this section, we discuss the imitation learning problem under context-distribution mismatch between the data and the online environment, namely, $\Poffline \neq \Ponline$. We prove that when the state transition function is independent of the unobserved context (even with access to the true transition function) the imitation learning problem is in general unsolvable, rendering the data useless for imitation. 

% To solve this setup, we define an ambiguity set of candidate optimal policies and prove its sufficiency for characterizing this set. For completeness, we also provide an algorithm in \Cref{appendix: confounded imitation} which calculates the ambiguity set and selects a robust policy. \gal{I feel there is too much of we will do this in section X. As if every second paragraph keeps promising somethign interesting later. Keep some, but remove others}

% In contrast, we show \gal{same. too much advertizing} that when the reward function is independent of the unobserved context, the optimal policy is indeed identifiable from the expert demonstrations. 

\subsection{No Covariate Shift: $\Ponline = \Poffline$}

We first consider the scenario in which no covariate shift is present between the offline data and the online environment, i.e., $\Ponline = \Poffline$. We begin by defining the marginalized ambiguity set, a central component of our work.
\begin{definition}[Ambiguity Set]
    For a policy $\pi \in \Pi$, we define the set of all deterministic policies that match the marginalized stationary distributions of $\pi$ by
    \begin{align*}
        \equivset{\pi} = \brk[c]*{\pi' \in \Pi_{\text{det}} : \donline^{\pi'}(s, a) = \doffline^{\pi}(s, a) \quad  \forall \, s \in \s, a \in \A}.
    \end{align*}
\label{def: ambiguity set}
\end{definition}

Recall that, in general, $\pi^* \in \Pi^*_\M$ may depend on the context $x \in \X$. Therefore, the set $\equivset{\pi^*}$ corresponds to all deterministic policies that cannot be distinguished from $\pi^*$ based on the observed expert data. The following theorem shows that for any policy $\pi^* \in \Pi^*_\M$ and any policy $\pi_0 \in \equivset{\pi^*}$, one could design a reward function $r_0$, for which $\pi_0$ is optimal, while the set $\equivset{\pi^*}$ is indiscernible from $\equivset{\pi_0}$, i.e., $\equivset{\pi^*}=\equivset{\pi_0}$  (see \Cref{appendix: missing proofs} for proof). In other words, $\equivset{\pi^*}$ is the smallest set of candidate optimal policies.

\begin{restatable}{theorem}{ambiguitythm}[Sufficiency of $\equivset{\pi^*}$]
\label{thm: ambiguity uniqueness}
    Assume $\Poffline \equiv \Ponline.$ Let $\pi^* \in \Pi^*_{\M}$ and let ${\pi_0 \in \equivset{\pi^*}}$. Then, $\equivset{\pi^*} = \equivset{\pi_0}$. Moreover, if $\pi_0 \neq \pi^*$, then there exists $r_0$ such that $\pi_0 \in \Pi^*_{\M_0}$ but $\pi^* \notin \Pi^*_{\M_0}$, where ${\M_0 = (\s, \A, \X, P, r_0, \Ponline, \nu, \gamma)}$.
\end{restatable}

The above theorem shows that any policy in $\equivset{\pi^*}$ is a candidate optimal policy, yet without knowing the context the expert used, no policy in $\equivset{\pi^*}$ can be ruled out (as they all have identical marginalized stationary distributions). Hence, the imitation solution is uniquely defined by the set $\equivset{\pi^*}$. Such ambiguity can result in selection of a suboptimal or even catastrophic policy. Nevertheless, as we show in the following proposition, acting uniformly w.r.t. $\equivset{\pi^*}$ is better than the worst policy in the set, i.e., robust to the ambiguity set (see \Cref{appendix: missing proofs} for proof).

\begin{restatable}{proposition}{ambiguityselectthm}
\label{thm: ambiguity policy selection}
    Define the mean policy
    $
        \bar{\pi}(a|s,x) = 
        \frac{
            \sum_{\pi \in \equivset{\pi^*}} d^{\pi}(s,a,x)
        }
        {
            \sum_{\pi \in \equivset{\pi^*}}  \sum_{a'} d^{\pi}(s,a',x)
        },
    $
    and denote ${\alpha^* = \frac{\abs{\Pi^*_\M \bigcap \equivset{\pi^*}}}{\abs{\equivset{\pi^*}}} \in [0, 1]}$. Then,
    $
        v_\M(\bar{\pi})
        \geq
        \alpha^* v^* + 
        (1-\alpha^*) \min_{\pi \in \equivset{\pi^*}} v_{\M}(\pi).
    $
\end{restatable}
\begin{remark} Note that
$\bar{\pi}$ is generally not the average policy $\frac{1}{\abs{\equivset{\pi^*}}}\sum_{\pi \in \equivset{\pi^*}} \pi(a | s,x)$. 
\end{remark}
\begin{remark}
In an episodic setting, $\bar{\pi}$ can be estimated by uniformly sampling a policy $\pi \in \equivset{\pi^*}$ at the beginning of the episode, and playing it until the environment terminates. 
\end{remark}

We provide a practical algorithm in \Cref{appendix: confounded imitation} which calculates the ambiguity set $\equivset{\pi^*}$, and returns $\bar{\pi}$ of \Cref{thm: ambiguity policy selection}, with computational guarantees, showing that $\bar{\pi}$ is returned after exactly $\abs{\equivset{\pi^*}}$ iterations. In the next subsection we analyze a more challenging scenario, for which $\Ponline \neq \Poffline$. In this case, $\equivset{\pi^*}$ may not be sufficient for the imitation problem.

\subsection{With Covariate Shift:  $\Ponline \neq \Poffline$}
\label{section: covariate shift}

Next, we assume covariate shift exists between the online environment and the expert data, i.e., ${\Ponline \neq \Poffline}$. Particularly, without further assumptions on the extent of covariate shift, we show two extremes of the problem. In \Cref{thm: impossibility} we prove that whenever the \emph{transitions} are independent of the context, the data cannot in general be used for imitation. In contrast, in \Cref{thm: context free reward} we prove that, whenever the \emph{reward} is independent of the context, the imitation problem can be efficiently solved. 

\begin{figure}[t!]
\centering
\includegraphics[width=0.979\linewidth]{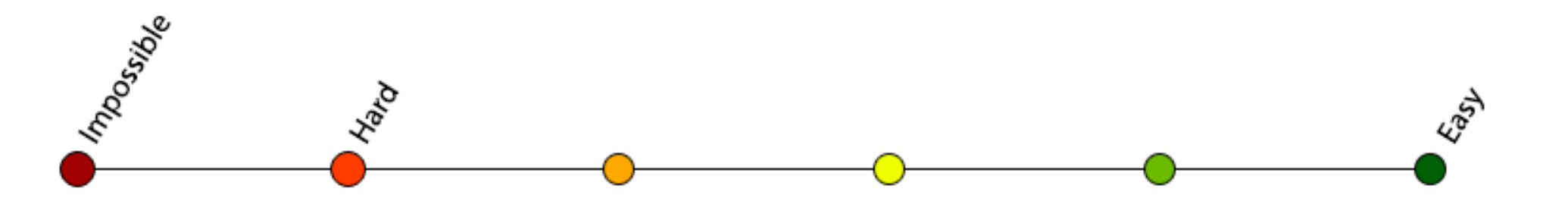}
\begin{tikzpicture}
\node[text width=3cm, align=center] at (-5.5,-1.5) {\centering \small Context-Independent \\ Transition (\Cref{thm: impossibility})};
\node[text width=3cm, align=center] at (-0,-1.5) {\centering \small Bounded Confounding \\ (\Cref{appendix: bounded confounding})};
\node[text width=3cm, align=center] at (5.4,-1.5) {\centering \small Context-Independent Reward \\ (\Cref{thm: context free reward})};
\end{tikzpicture}
% \vspace*{-0.5cm}
\caption{\small A spectrum for the difficulty of confounded imitation with covariate shift.}
\label{fig: spectrum}
\end{figure}
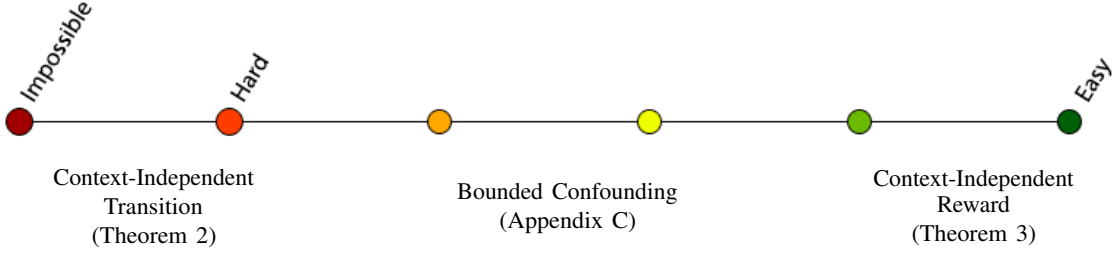

% In general, it can be shown that under mild assumptions, whenever $P(s'|s,a,x)$ is independent of the context $x$, the imitation problem is \Assaf{Did you define what does possible and impossible mean exactly? if not maybe we should} impossible. 

Clearly, if $\text{Supp}(\Ponline) \not\subseteq \text{Supp}(\Poffline)$\footnote{For a distribution $\mathbb{P}$ we denote by $\text{Supp}(\mathbb{P})$ the support of $\mathbb{P}$.} then there exists ${x \in \text{Supp}(\Ponline)}$ for which $\pi^*$ is not identifiable from the expert data\footnote{We define non-identifiability in \Cref{def: non-identifiablity}. We use a similar notion of identifiability as in \citet{pearl2009causal}}. We therefore assume throughout that ${\text{Supp}(\Ponline) \subseteq \text{Supp}(\Poffline)}$. 
We begin by defining the set of non-identifiable policies as those that cannot be distinguished from their respective stationary distributions without information on $\Poffline$.
\begin{definition}
\label{def: non-identifiablity}
    We say that $\brk[c]*{\pi_i}_{i=1}^k$ are non-identifiable policies if there exist $\brk[c]*{\rho_i}_{i=1}^k$ such that ${d_{\rho_i}^{\pi_i}(s,a) = d_{\rho_j}^{\pi_j}(s,a)}$ for all $i \neq j$.
\end{definition}
Next, focusing on catastrophic policies (recall \Cref{eq: catastrophic policy}), we define catastrophic expert policies as those which could be either optimal or catastrophic under $\Ponline$ for different reward functions.
\begin{definition}
    We say that $\brk[c]*{\pi_i}_{i=1}^k$ are catastrophic expert policies if there exist $\brk[c]*{r_i}_{i=1}^k$ such that for all $i$, ${\pi_i \in \Pi^*_{\M_i}}$, and $\exists j \in [k], j \neq i$ such that $\pi_i \in \Pi^\dagger_{\M_j}$, where ${\M_j = (\s, \X, \A, P, r_j, \Ponline, \nu, \gamma)}$.
\end{definition}
Using the fact that both $\Poffline$ and $r$ are unknown, the following theorem shows that whenever $P(s'|s,a,x)$ is independent of $x$, one could find two policies which are non-identifiable, catastrophic expert policies (see \Cref{appendix: missing proofs} for proof). In other words, in the case of context-independent transitions, without further information on $\Poffline$ or $r$ the expert data is useless for imitation. Furthermore, attempting to imitate the policy using the expert data could result in a catastrophic policy.

\begin{restatable}{theorem}{impossibilitythm}[Catastrophic Imitation]
\label{thm: impossibility}
    Assume ${\abs{\X} \geq \abs{\A}}$, and ${P(s'|s,a,x) = P(s'|s,a,x')}$ for all $x, x' \in \X$. Then $\exists \pi_{e,1}, \pi_{e,2}$ s.t. $\brk[c]*{\pi_{e,1}, \pi_{e,2}}$ are non-identifiable, catastrophic expert policies.
\end{restatable}

While \Cref{thm: impossibility} shows the impossibility of imitation for context-free transitions, whenever the reward is independent of the context, the imitation problem becomes feasible. In fact, as we show in the following theorem, for context-free rewards, any policy in $\equivset{\pi^*}$ is an optimal policy.

\begin{restatable}{theorem}{contextfreereward}[Sufficiency of Context-Free Reward]
\label{thm: context free reward}
    Assume ${\text{Supp}(\Ponline) \subseteq \text{Supp}(\Poffline)}$ and ${r(s,a,x) = r(s,a,x')}$ for all $x, x' \in \X$. Then ${\equivset{\pi^*} \subseteq \Pi^*_\M}$.
\end{restatable}

Theorems~\ref{thm: impossibility} and~\ref{thm: context free reward} suggest that the hardness of the imitation problem under covariate shift lies on a wide spectrum (as depicted in \Cref{fig: spectrum}). While dependence of the transition $P(s'|s,a,x)$ on $x$ provides us with information to identify $x$ in the expert data, the dependence of the reward $r(s,a,x)$ on $x$ increases the degree of confounding in the imitation problem. Both of these results are concerned with arbitrary confounding. For the interested reader, we further analyze the case of bounded confounding in \Cref{appendix: bounded confounding}. We also demonstrate the effect of bounded confounding in \Cref{section: experiments}. 

In the following section, we show that, while arbitrary confounding may result in catastrophic results for the imitation learning problem, when coupled with reward, one can still make use of the expert data.

\begin{algorithm}[t!]
\caption{{RL using Expert Data with Unobserved Confounders (Follow the Leader)}}
\label{algo: ftrl}
% \label{algo: partial imitation no shift}
\begin{algorithmic}[1]
% \small
\STATE {\bf input:} Expert data with missing context $\D^*$, ${\lambda > 0}$, policy optimization algorithm \texttt{ALG-RL}
\STATE {\bf init:} Policy $\pi^0$
\FOR{$k = 1, \hdots$}
    \STATE $\rho_s \gets \arg\min_\rho D_{KL}(\donline^{\pi^{k-1}}(s,a) || d_\rho^{\pi^*}(s,a))$
    \STATE $g^k \gets \frac{1}{k}\brk*{g^{k-1} + \expect*{s,a \sim \donline^{\pi^{k-1}}}{\frac{1}{d_{\rho_s}^{\pi^*}(s,a)}}}$ (FTL Cost Player)
    \STATE $\pi^k \gets \text{\texttt{ALG-RL}}(r(s,a,x) - \lambda g^k(s,a))$
\ENDFOR
\end{algorithmic}
\end{algorithm}

\section{Using Expert Data with Unobserved Confounders for RL}
\label{section: rl}

In the previous section we showed sufficient conditions under which imitation is possible, with and without covariate shift. When covariate shift is present, but unknown, the imitation learning problem may be hard, or even impossible (see \Cref{thm: impossibility}, catastrophic imitation). We ask, had we had access to the reward function, would the expert data be useful under arbitrary covariate shift? In this section we show that expert data with unobserved confounders can be used to converge to an expert policy, even when arbitary covariate shift is present. In our experiments (\Cref{section: experiments}) we empirically show that using our method can also improve overall performance.

We view the confounded expert data as side information to the RL problem. Specifically, we assume access to the true reward signal in the online environment and wish to leverage the offline expert data to aid the agent in converging to an optimal policy. To do this, we define an optimization problem that maximizes the cumulative reward, while minimizing an $f$-divergence (e.g., KL-divergence, TV-distance, $\chi^2$-divergence) of stationary distributions in $\equivset{\pi^*}$,
\begin{align}
    \max_{\pi \in \Pi} \expect*{x \sim \Ponline, s,a \sim d^\pi(s,a | x)}{r(s,a,x)} - \lambda D_f(\donline^{\pi}(s,a) || \doffline^{\pi^*}(s,a)).
    \label{eq: orig problem}
    \tag{P1}
\end{align}
Here, $\lambda > 0$ and $D_f$ is the $f$-divergence, where $f$ is a convex function $f: (0, \infty) \mapsto \R$. The solution to Problem~(\ref{eq: orig problem}) is an optimal policy ${\pi^* \in \Pi^*_\M}$ as long as $\Ponline \equiv \Poffline$. Rewriting $D_f$ using its variational form (see \Cref{appendix: distribution matching} for background on the variational form of $f$-divergences), we get the following equivalent optimization problem, motivated by \citet{nachum2019algaedice}:
\begin{align}
    \max_{\pi\in \Pi} \min_{g: \s \times \A \mapsto \R} 
     ~~\expect*{x \sim \Ponline, s,a \sim d^\pi(s,a | x)}{r(s,a,x) + \lambda g(s,a)} 
    - \lambda \expect*{s,a \sim \doffline^{\pi^*}(s,a)}{f^*(g(s,a))},
    \label{eq: max min problem}
    \tag{P1b}
\end{align}
where $f^*$ is the convex conjugate of $f$, i.e., ${f^*(y) = \sup_x xy - f(y)}$. 

% Notice that strong duality holds for the problem in Problem~(\ref{eq: max min problem}). Therefore, it can be solved by iterations similar to \citet{ho2016generative}, using alternate gradient updates w.r.t. $g(s,a)$ and a reinforcement learning agent with an augmented reward $r(s,a,x) + \lambda g(s,a)$. The inner minimization problem can be solved using samples from $\donline^\pi(s,a)$ (the current policy) and $\doffline^{\pi^*}(s,a)$ (the expert data). 

Unfortunately, when covariate shift exists (i.e., $\Ponline \neq \Poffline$), Problems~(\ref{eq: orig problem}) and~(\ref{eq: max min problem}) are not ensured to converge to an optimal policy (\Cref{thm: impossibility}). Instead, we propose to reformulate Problem~(\ref{eq: max min problem}) using a distribution $\rho_s$ which minimizes the $f$-divergence, as follows,
\begin{align}
    \max_{\pi\in \Pi} 
    \min_{\substack{g: \s \times \A \mapsto \R \\ \rho_s \in \mathcal{B}\brk*{\X}}} 
     ~~\expect*{x \sim \Ponline, s,a \sim d^\pi(s,a | x)}{r(s,a,x) + \lambda g(s,a)} 
    - \lambda \expect*{s,a \sim d_{\rho_s}^{\pi^*}(s,a)}{f^*(g(s,a))}.
    \label{eq: max min min problem}
    \tag{P2}
\end{align}
Here, $\mathcal{B}\brk*{\X}$ denotes the set of probability measures on the Borel sets of $\X$, and
${
    d_{\rho_s}^{\pi^*}(s,a) = \expect*{x \sim \rho_s}{d^{\pi^*}(s,a \mid x)}.}
$
Indeed, whenever ${\text{Supp}(\Ponline) \subseteq \text{Supp}(\Poffline)}$, we have that ${\brk*{\pi, \rho_s} = \brk*{\pi^*, \Ponline}}$ is a solution to Problem~(\ref{eq: max min min problem}). That is, unlike Problems~(\ref{eq: orig problem}) and~(\ref{eq: max min problem}), Problem~(\ref{eq: max min min problem}) can achieve an optimal solution to the RL problem which still uses the expert data.

\begin{algorithm}[t!]
\caption{{RL using Expert Data with Unobserved Confounders (Online Gradient Descent)}}
\label{algo: ogd}
% \label{algo: partial imitation no shift}
\begin{algorithmic}[1]
% \small
\STATE {\bf input:} Expert data with missing context, ${\lambda, B, N > 0}$, policy optimization algorithm \texttt{ALG-RL}
\STATE {\bf init:} Policy $\pi^0$, bonus reward network $g_\theta$ 
\FOR{$k = 1, \hdots$}
    \STATE $\rho_s \gets \arg\min_\rho D_f(\donline^{\pi_{k-1}}(s,a) || d_\rho^{\pi^*}(s,a))$ 
    \FOR{$e = 1, \hdots N$}
        \STATE Sample batch $\brk[c]*{s_i, a_i}_{i=1}^B \sim \donline^{\pi_{k-1}}(s,a)$
        \STATE Sample batch $\brk[c]*{s_i^e, a_i^e}_{i=1}^B \sim d_{\rho_s}^{\pi^*}(s,a)$
        \STATE Update $g_\theta$ according to
        $~
            \nabla_\theta L(\theta) 
            =
            \frac{1}{B}\sum_{i=1}^B\nabla_\theta \brk[s]*{f^*(g_\theta(s_i^e, a_i^e)) - g_\theta(s_i, a_i)}
       $
    \ENDFOR
    \STATE $\pi^k \gets \text{\texttt{ALG-RL}}(r(s,a,x) - \lambda g_\theta(s,a))$
\ENDFOR
\end{algorithmic}
\end{algorithm}

\paragraph{Corrective Trajectory Sampling (CTS).}

Solving Problem~(\ref{eq: max min min problem}) involves an expectation over an unknown distribution, $d^{\pi^*}_{\rho_s}(s,a)$. Fortunately, $d^{\pi^*}_{\rho_s}(s,a)$ can be equivalently written as an expectation over trajectories in $\D^*$, rather than expectation over unobserved contexts, as shown by the following proposition (see \Cref{appendix: missing proofs} for proof):

\begin{restatable}{proposition}{samplingprop}[Trajectory Sampling Equivalence]
\label{prop: sampling equivalence}
Let $\rho_s^*$ which minimizes Problem~(\ref{eq: max min min problem}) for some $\pi \in \Pi, g: \s \times \A \mapsto \R$, and assume ${\text{Supp}(\Ponline) \subseteq \text{Supp}(\Poffline)}$. Then, there exists $p^n \in \Delta_n$ such that
    $
        d^{\pi^*}_{\rho_s^*}(s,a) = \lim\limits_{n\to \infty} \expect*{i \sim p^n}{(1-\gamma) \sum_{t=0}^\infty \gamma^t \mathbf{1}\brk[c]*{(s_t^i, a_t^i) = (s, a)}}.
    $
\end{restatable}

\Cref{prop: sampling equivalence} allows us to estimate the inner minimization problem over $\rho_s$ in Problem~(\ref{eq: max min min problem}) using trajectory samples. Particularly, we uniformly sample $k$ distributions $p^n_1, \hdots p^n_k$, where $p^n_j \in \Delta_n$, and then estimate
\begin{align}
\min_{\rho_s} D_f(\donline^{\pi} || d_{\rho_s}^{\pi^*})
\approx
\min_{j \in \brk[c]*{1, \hdots, k}}\brk[c]*{
    D_f
    \brk*{
        \donline^\pi(s,a) ~\Big|\Big|~
        \expect*{i \sim p^n_j}{\sum_{t=0}^\infty \gamma^t \mathbf{1}\brk[c]*{(s_t^i, a_t^i) = (s, a)}}
    }},
    \label{eq: sampling equivalence}
\end{align}
which can be estimated by using the variational form of $D_f$ (see \Cref{appendix: distribution matching}). We call this procedure Corrective Trajectory Sampling (CTS), as it uses complete trajectory samples to account for the unknown context distribution $\Poffline$.

\paragraph{Solving Problem~(\ref{eq: max min min problem}).} \Cref{algo: ftrl} provides an iterative procedure for solving the optimization problem in Problem~(\ref{eq: max min min problem}). It uses alternative updates of a cost player (line 5) and policy player (line 6). In line~5 the gradient of $D_{KL}$ w.r.t. $d^\pi$ is taken using a Follow the Leader (FTL) cost player to estimate the next bonus iterate. Finally, in line~6, an efficient, approximate policy optimization algorithm \texttt{ALG-RL} is executed using an augmented reward. The following theorem, provides convergence guarantees for \Cref{algo: ftrl} with an approximate best response RL-algorithm (see \Cref{appendix: missing proofs} for proof based on \citet{zahavy2021reward}).
\begin{restatable}{theorem}{samplecomplexity}
\label{thm: rl convergence}
    Let \texttt{ALG-RL} be an approximate best response player that solves the RL problem in iteration $k$ to accuracy $\epsilon_k = \frac{1}{\sqrt{k}}$. Then, \Cref{algo: ftrl} will converge to an $\epsilon$-optimal solution to Problem~(\ref{eq: max min min problem}) in  $\mathcal{O}\brk*{\frac{1}{\epsilon^4}}$ samples.
\end{restatable}
Notice that, while \Cref{thm: rl convergence} shows \Cref{algo: ftrl} converges to an optimal policy, it does not determine whether the expert data improves overall learning efficiency. We leave this theoretical question for future work. Nevertheless, in the following section we conduct extensive experiments to show that such data can indeed improve overall performance on various tasks. 

A drawback of \Cref{algo: ftrl} is that it needs to estimate the stationary distributions. A practical implementation of \Cref{algo: ftrl} using online gradient descent (OGD) is provided in \Cref{algo: ogd} -- the algorithm does not require approximate estimates of the stationary distributions, but rather, only the ability to sample from them. Similar to \Cref{algo: ftrl}, we use CTS (see \Cref{eq: sampling equivalence}) to estimate $\rho_s$ in line~4 according to some $f$-divergence. Here, samples are drawn from the current policy as well as samples from $\D^*$ (with CTS). We write $D_f$ in its variational form, and use a neural network representation for $g_\theta$. We then use the aforementioned samples to minimize the $f$-divergence using OGD. Finally, the policy is updated using \texttt{ALG-RL} and an augmented reward.

\begin{figure}[t]
\centering
\includegraphics[width=0.44\linewidth]{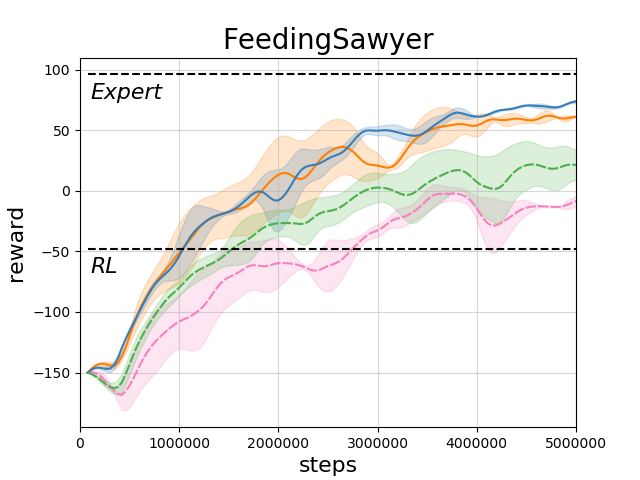}
\includegraphics[width=0.44\linewidth]{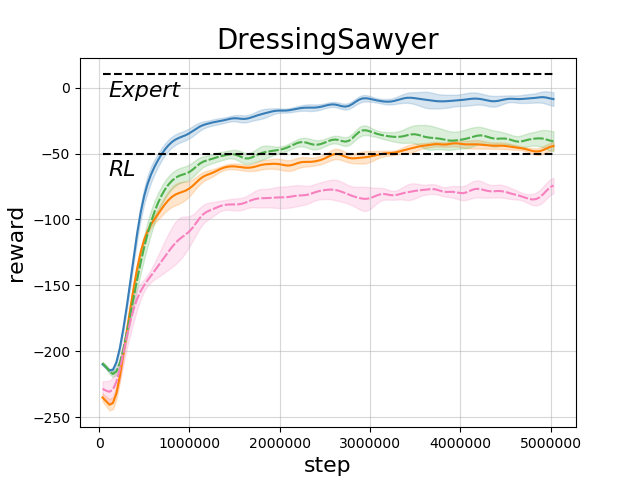}
\includegraphics[width=0.44\linewidth]{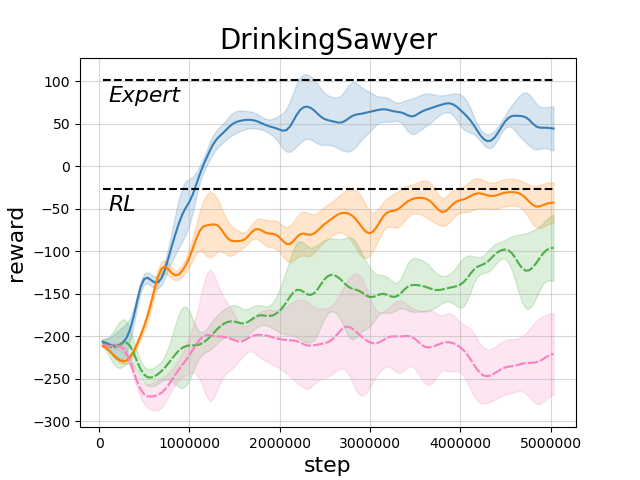}
\includegraphics[width=0.44\linewidth]{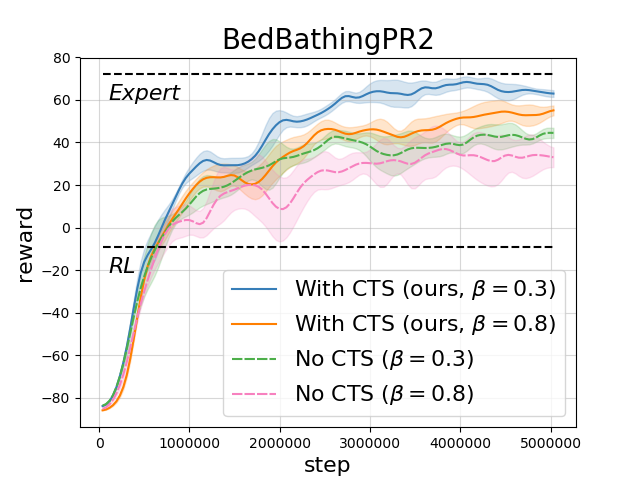}
\caption{\small Plots compare training curves of using CTS vs. normal sampling of expert data for small ($\beta = 0.3$) and large ($\beta = 0.8$) covariate shift bias in four assistive-healthcare tasks. Dashed black lines show expert and RL (without data) scores. Runs were averaged over 5 seeds. Legend is shared across all plots.}
\label{fig: assistive}
\end{figure}

\section{Experiments}
\label{section: experiments}

We tested our proposed approach for using expert data with hidden confounding in recommender-system and assistive-healthcare environments. For all our experiments we used $\chi^2$-divergence as our choice of $f$-divergence, as we found it to work best. Comparison to other divergences is provided in \Cref{fig: recsim} (left). We used PPO \citep{schulman2017proximal} implemented in RLlib \citep{liang2018rllib} for both the imitation as well as RL settings. We include specific choice of hyperparameters and an exhaustive overview of further implementation details in \Cref{appendix: implementation details}.

\textbf{Assistive Healthcare.} 
Consider the challenge of providing physical assistance to disabled persons. A recently proposed set of tasks for assistive-healthcare, simulating autonomous robots as versatile caregivers \citep{erickson2020assistive}. Each task has a unique goal, affected by both the physical world as well as the patient specific preferences and disabilities. 

We tested our algorithm on four tasks: feeding, dressing, bathing, and drinking. In these, we used the following features to define user context: gender, mass, radius, height, patient impairment, and patient preferences. The patient's mass, radius, and height distributions were dependent on gender. The patient's impairment was given by either limited movement, weakness, or tremor (with sporadic movement). Finally, the patient's preferences were affected by the velocity and pressure of touch forces applied by the robot. For the context distribution $\Ponline$ we used the default values as provided by the original environment. To enforce a distribution shift in the expert data, we shifted each distribution randomly with an additive factor $\beta \cdot \tilde{d_x}$, where $\beta \in [0, 1]$, and $ \tilde{d_x}$ was a random distribution chosen from a set of shifting distributions (see \Cref{appendix: implementation details}). Here $\beta$ corresponds to the covariate shift strength. The expert data was generated by a fixed policy trained using a dense reward function. A sparse reward signal was used for executing our experiments with the confounded expert data. For further details, we refer the reader to \Cref{appendix: implementation details}. 

\Cref{fig: assistive} depicts results for executing \Cref{algo: ogd} on four assistive-gym environments with various covariate shift strengths. As evident in most of the enviornments, covariate shift strongly affected overall performance. Particularly in the feeding, drinking, and dressing environments, the success of reaching the goal (i.e., spoon to mouth, cup to mouth, and sleeve to hand) was highly affected by the degree of covariate shift. This is due to the changing distribution of size, movement, and preferences of the patient, and thus of the goal. Nevertheless, in all environments, using the expert data (with and without CTS) was found to help induce better policies than executing the same RL algorithm without expert data. This suggests that expert data can assist in improving overall RL performance, yet correcting for covariate shift may significantly improve it in these domains.

\begin{figure}[t!]
\centering
\includegraphics[width=0.325\linewidth]{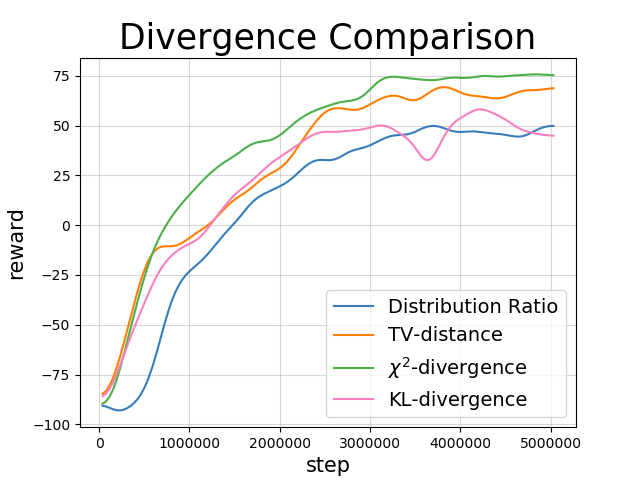}
\includegraphics[width=0.325\linewidth]{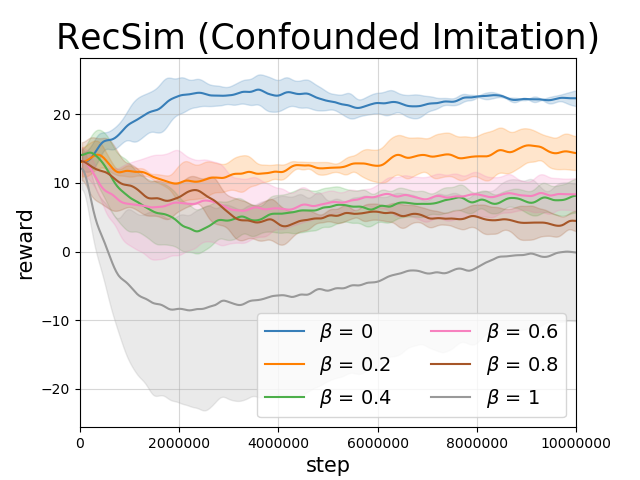}
\includegraphics[width=0.325\linewidth]{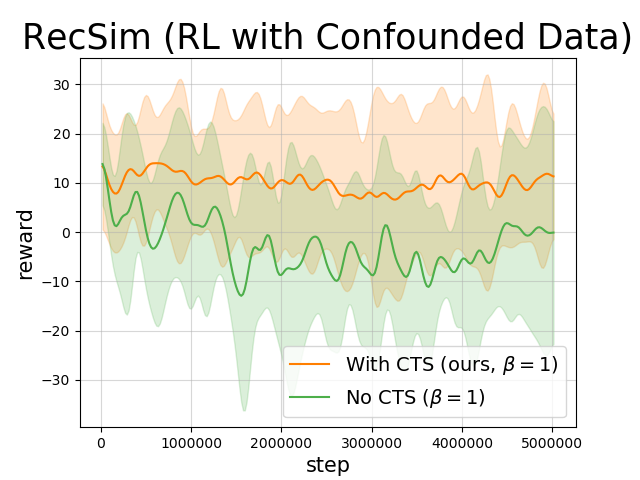}
\caption{\small Left plot shows comparison of different choices of $f$-divergences for pure imitation (without reward and without covariate shift) on the BedBathingPR2 environment. Middle plot depicts execution of imitation with hidden confounding (without reward) for different levels of covariate shift. Right plot compares our CTS correction on the RecSim environment with strong covariate shift bias. All runs were averaged over 5 seeds. }
\label{fig: recsim}
\end{figure}

\textbf{Recommender Systems.} 
In practical recommender systems, sequential interaction with users presents a great challenge for optimizing user long-term engagement and overall satisfaction \citep{ie2019recsim}. Leveraging expert data collected using, e.g., surveys to users, may greatly benefit future solutions. Because features are repeatedly added to these systems, full information in the data is rarely available. Here, we use the recently proposed RecSim Interest Evolution environment \citep{ie2019recsim}, simulating sequential user interaction with a slate-based recommender system. The environment consists of a document model for sampling documents, a user model for defining a distribution over user context features, and a user choice model, which defines the intent of the user based on observable document features and the user's sampled context (e.g., personality, satisfaction, interests, demographics, and other behavioral features such as session length or visit frequency). 

We used a slate of 10 documents and a user context of dimension 20. To test the severity of the implications of \Cref{thm: impossibility} in the confounded imitation setting, we used a user-model sampled from a Beta-distribution. Specifically, for the expert data the user context features $x = (x_0, \hdots, x_{19})$ were sampled from a Beta-distribution, where $x_i \sim Beta(\alpha_i, 4)$, and $\alpha_i = 1.5 + \frac{8.5}{19}i$. In contrast, the online environment features were sampled from a shifted Beta-distribution with ${\alpha_i = (1-\beta)\brk*{1.5 + \frac{8.5}{19}i} + \beta\brk*{10 - \frac{8.5}{19}i}}$, where $\beta \in [0,1]$ defined the shift strength. While the original environment used a uniform distribution to generate user contexts, the Beta-distribution let us analyze severe forms of covariate shifts, testing the limits of our results in \Cref{section: imitation,section: rl}. Still, we emphasize that both the original (uniform) distribution and the Beta distribution are not necessarily realistic user distributions. 

\Cref{fig: recsim}(a) depicts the effect of increased covariate shift on imitation in the RecSim environment with a dataset of 100 expert trajectories (generated by an optimal policy that had access to the full context). Without covariate shift ($\beta = 0$) an optimal score is achieved, and as $\beta$ increases, performance decreases. Particularly, as the mirrored distribution is reached ($\beta = 1$), a catastrophic policy is reached. While the imitator ``believes" to have reached an optimal policy, it has in fact reached a catastrophic one, as shown by the orange plot. Conversely, \Cref{fig: recsim}(b) depicts the benefit of using confounded expert data in the RL setting, i.e. when an online reward signal is available. Though strong confounding is present, the agent is capable of leveraging the data to improve overall learning performance.

\section{Related Work}
\label{section: related work}

\textbf{Imitation Learning.} The imitation learning problem has been extensively studied in both the fully offline \citep{pomerleau1989alvinn,bratko1995behavioural} as well as online setting \citep{ho2016generative,fu2018learning,kim2018imitation,brantley2019disagreement}. Specific to our work are GAIL \citep{ho2016generative}, AIRL \citep{fu2017learning}, and DICE \citep{kostrikov2019imitation}, which use distribution matching methods. Our work generalizes these settings to imitation with hidden confounders.

\textbf{Reinforcement Learning with Expert Data.} Much work has revolved on leveraging offline data for RL. Recently, offline RL \citep{levine2020offline} has shown great improvement over regular offline imitation techniques \citep{kumar2020conservative,kostrikov2021offline,tennenholtz2021gelato,fujimoto2021minimalist}. In the online RL setting, the combination of offline data to improve RL efficiency has shown great success \citep{nair2020accelerating}. KL-regularized techniques \citep{peng2019advantage,siegel2019keep} as well as DICE-based algorithms \citep{nachum2019algaedice} have also shown efficient utilization of offline data. Our work generalizes the latter to the confounded setting.

\textbf{Intersection of Causal Inference and Imitation Learning.} Closely related to our work is that of \citet{zhang2020causal}. There, the authors suggest a notion of imitability, showing when observational data can help identify a policy under some partially observed structural causal model. Our work provides an alternative perspective on the problem. In contrast to their work, we rely on concurrent imitation approaches (i.e. stationary distribution matching techniques) and importantly, allow access to the online environment. Furthermore, we provide guarantees and practical algorithms for both the imitation and RL settings. We refer the reader to \Cref{appendix: relation to causal inference} for a specific interpretation of our framework in CI terminology, from a perspective of stochastic interventions in a Structural Causal Model.

Another intersection with causal inference discusses the problem of causal confusion in imitation \citep{de2019causal}. Causal confusion is concerned with the problem of nuisances in \emph{observed} confounded data due to an unknown causal structure. These  ``causal misidentifications'' can lead to spurious correlations and catastrophic failures in generalization. In contrast, our work discusses the orthogonal problem of hidden confounders with possible covariate shift.

\textbf{Intersection of Causal Inference and Reinforcement Learning.} Previous work has analyzed the problem of optimal control from logged data with unboserved confounders \citep{lattimore2016causal}, as well as utilizing (non-expert) confounded data for online interactions \citep{tennenholtzbandits}. Much work has revolved around the reinforcement learning setup with access to (non-expert) confounded data \citep{zhang2019near,wang2020provably}. Other work has considered the problem of off-policy evaluation from confounded data \citep{tennenholtz2020off,oberst2019counterfactual,kallus2020confounding}. Our work is focused on leveraging \emph{expert} data with hidden confounders and possible covariate shift in both the imitation and the RL settings.

\section{Conclusion}

This work presented and analyzed the problem of using expert data with hidden confounders for both the imitation and RL settings. We showed that covariate shift of hidden confounders between the expert data and the online environment can result in learning catastrophic policies, rendering imitation learning hard, or even impossible (\Cref{thm: impossibility}). In addition, we showed that when a reward is provided, using the expert data is still possible under arbitrary hidden covariate shift (\Cref{thm: rl convergence}). We proposed new algorithms for tackling this problem using corrective trajectory sampling (CTS). Our empirical evaluation demonstrates our results and suggests that taking the possibility of unknown covariate shift into account may significantly improve overall performance.

\bibliography{main}
\bibliographystyle{unsrtnat}

\newpage
\onecolumn
\appendix
\section*{Appendix}

\begin{table*}[t!]
  \small\centering
  \hspace*{-1cm}
  \begin{tabular}{lcc}

    \toprule

    {\bfseries Distribution Matching} & {\bfseries Equivalent Representation} & {\bfseries Comments}\\
    \midrule\midrule[.1em]

    Distribution Ratio
      & $
        \sup_{g: \s \times \A \mapsto (0,1)}
        \expect*{s, a \sim \doffline^{\pi'}}{\log\brk*{g(s,a)}}
        +
        \expect*{s, a \sim \donline^\pi}{{\log\brk*{1 - g(s,a)}}}
        $
      & GAIL \\
      ~ & ~ & \citep{ho2016generative}  \\
    \midrule[.1em] 
    
    KL-divergence \qquad \qquad \qquad    
      & $
        \sup_{g: \s \times \A \mapsto \R}
        \expect*{s, a \sim \donline^\pi}{g(s, a)}
        -
        \log \expect*{s, a \sim \doffline^{\pi'}}{e^{g(s, a)}}
        $
      &  Donsker-Varadhan Representation \\
      ~ & ~ & \citep{kostrikov2019imitation}  \\
    \midrule[.1em] 

    $\chi^2$-divergence
      & $
        \sup_{g: \s \times \A \mapsto \R}
        2\expect*{s, a \sim \donline^\pi}{g(s, a)}
        -
        \expect*{s, a \sim \doffline^{\pi'}}{g^2(s, a)}
        $
      & Variational Representation \\
    ~ & ~ & of $f$-Divergence  \\
    \midrule[.1em] 
     
     TV Distance
      & $
        \sup_{|g| \leq \frac{1}{2}}
        \expect*{s, a \sim \donline^\pi}{g(s, a)}
        -
        \expect*{s, a \sim \doffline^{\pi'}}{g(s, a)}
        $
      & Variational Representation \\
      ~ & ~ & of $f$-Divergence  \\
    \midrule[.1em] 
    \bottomrule
  \end{tabular}
  \caption{ \label{table: imitation methods} Different distribution matching techniques and their equivalent representations. }
\end{table*}

\section{Background: Distribution Matching for Imitation Learning}
\label{appendix: distribution matching}

A common approach used in (non-confounded) imitation learning is matching the policy's stationary distribution $\donline^\pi$ to the offline target distribution $\doffline^{\pi^*}$. Consider a source distribution ${p \in \Delta_N}$ and target distribution ${q \in \Delta_N}$. GAIL \citep{ho2016generative} uses the distribution ratio objective $log\brk*{p / q}$, which can estimated using a GAN-like objective 
$
    D_{R}(p || q) 
    =
    \sup_{g: \Z \mapsto (0,1)}
        \expect*{p}{\log\brk*{g(z)}} 
        +
        \expect*{q}{{\log\brk*{1 - g(z)}}}
$,
to match the distribution $p$ to $q$.

This technique can be generalized to $f$-divergences \citep{csiszar2004information,liese2006divergences,kostrikov2019imitation,ke2020imitation}. Specifically, we wish to minimize a discrepancy measure from $p$ to $q$, namely $\min_{p \in \mathcal{K}} D(p || q)$. For a convex function ${f: [0: \infty) \mapsto \R}$, the $f$-divergence of $p$ from $q$ is defined by ${D_f(p || q) = \expect*{q}{f\brk*{\frac{p}{q}}}}$. DICE \citep{kostrikov2019imitation} uses the variational representation of the $f$-divergence,
\begin{align*}
    D_f(p || q)
    =
    \sup_{g: \Z \mapsto \R}
        \expect*{p}{g(z)}
        -
        \expect*{q}{f^*(g(z))},
\end{align*}
where $f^*$ is the Fenchel conjugate of $f$ defined by ${f^*(y) = \sup_x xy - f(y)}$. The convex conjugate has closed form solutions for the total variation distance, KL-divergence, $\chi^2$-divergence, Squared Hellinger distance, Le Cam distance, and Jensen-Shannon divergence. Using the variational representation of the $f$-divergence we can estimate $D_f$ using samples from $p$ and $q$. \Cref{table: imitation methods} presents examples of various $f$-divergences and their respective dual formulation. We also add the distribution ratio for comparison to the table, though it is not an $f$-divergence.

\section{Confounded Imitation - Algorithm and Convergence Guarantees}
\label{appendix: confounded imitation}

\begin{figure}[t!]
\centering
\includegraphics[width=0.9\linewidth]{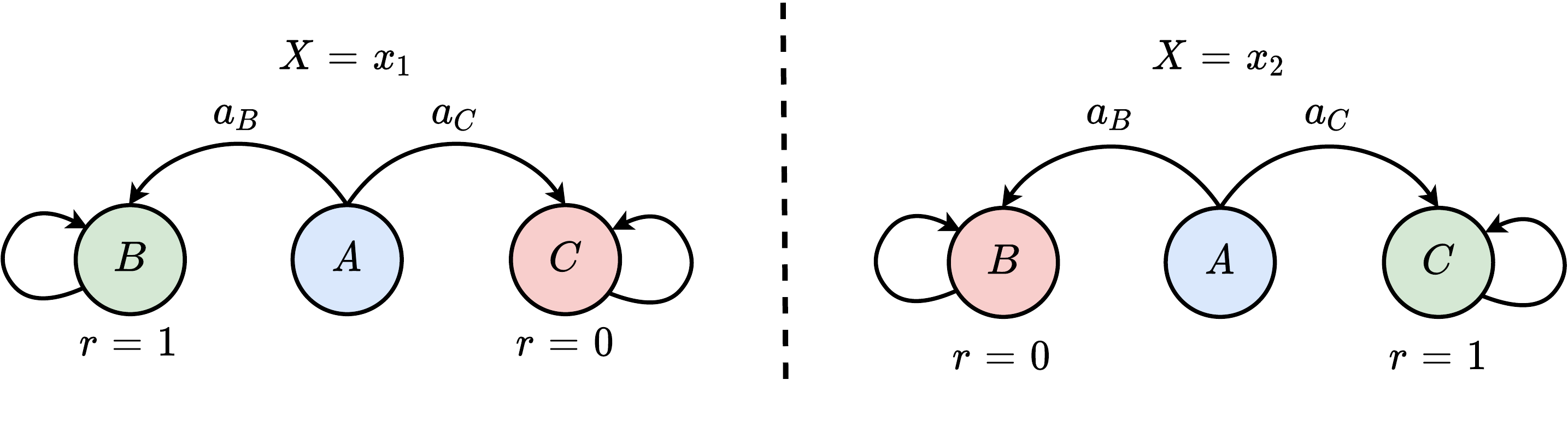}
\caption{A contextual MDP with state space $\s = \brk[c]*{A, B, C}$, action space ${\A = \brk[c]*{a_B, a_C}}$ and context space $\X = \brk[c]*{x_1, x_2}$. We assume $\nu(A | x) = 1$ for all $x \in \X$.  The actions $a_B, a_C$ transition the agent to states $B, C$, respectively, after which the agent receives a reward $r \in \brk[c]*{0, 1}$ depending on the context. We assume $B, C$ are sink states.
}
\label{fig:3_states}
\end{figure}

\subsection{A Toy Example}

To gain intuition, we start with a simple toy example. Consider the three-state example depicted in \Cref{fig:3_states}. Here, the environment initiates at state $A$ w.p.~1, after which the agent can choose to (deterministicaly) transition to state $B$ or $C$. The agent then receives a reward depending on the context. The optimal policy is given by
${
    \pi^*(a | s, x)
    =
    \mathbf{1}\brk[c]*{a = a_B, x = x_1} 
    +
    \mathbf{1}\brk[c]*{a = a_C, x = x_2}}
$
for ${s = A}$, and any action is optimal for ${s \neq A}$. Without loss of generality we assume $\pi^*(a_B | B, x) = \pi^*(a_C | C, x) = 1$. We turn to analyze the marginalized stationary distribution, which uniquely defines the set of optimal policies \citep{puterman2014markov}. Denoting $\Poffline(x_1) = \rho$, we have that
${
    \doffline^{\pi^*}(s,a)
    =
    \rho d^{\pi^*}(s,a | x_1) + (1-\rho)d^{\pi^*}(s, a | x_2).}
$
Then,
${
\doffline^{\pi^*}(s,a)
=
\brk*{1- \gamma}\mathbf{1}\brk[c]*{s = A} +
\rho \gamma\mathbf{1}\brk[c]*{s = B, a=a_B} +
(1-\rho) \gamma \mathbf{1}\brk[c]*{s = C, a=a_C}.}
$

\paragraph{No Covariate Shift.}  
Suppose ${\Ponline =  \Poffline}$, and $\rho = \frac{1}{2}$. Trivially ${\doffline^{\pi^*}(s,a) = \donline^{\pi^*}(s,a)}$. We define the (suboptimal) policy
\begin{align}
\label{eq: suboptimal policy}
    \pi_0(a | A, x)
    &=
    1 - \pi^*(a | A, x) \quad, a \in \A, x \in \X.
\end{align}
It can be verified that $\doffline^{\pi^*}(s,a) = \donline^{\pi_0}(s,a)$ still holds, yet $\pi_0$ is catastrophic (\Cref{eq: catastrophic policy}) with value zero. A question arises: can we show that $\pi_0$ is a suboptimal policy given access to the expert data (i.e., access to $\doffline^{\pi^*}(s,a)$) and a forward model $P(s' | s, a, x)$? 

Unfortunately, one cannot prove that $\pi_0$ is suboptimal. Informally, notice that $\pi_0$ is an optimal policy for an alternative reward function, 
${
    r_0(s, a, x) = 1 - r(s,a,x),
}$
yet is catastrophic w.r.t. the true reward $r$. Indeed, since $r$ is unknown and $\donline^{\pi_0}(s,a) = \donline^{\pi^*}(s,a)$, we cannot reject $r_0$ (i.e., we cannot conclude that $r_0$ is not the true reward). In other words, one cannot use the data to differentiate which of $\brk[c]*{\pi_0, \pi^*}$ is the optimal policy.

\paragraph{With Covariate Shift.} Next, assume $\Ponline \neq \Poffline$, and define $\pi_0$ as in \Cref{eq: suboptimal policy}. Let $\widetilde{\Poffline} = 1 - \Poffline$ and recall that $\Poffline(x_1) = \rho$. Then, we have that
\begin{align*}
    d^{\pi_0}_{\widetilde{\Poffline}}(s,a)
    &=
    (1-\rho) d^{\pi_0}(s,a | x_1) + \rho d^{\pi_0}(s, a | x_2)
    =
    (1-\rho)d^{\pi^*}(s, a | x_2) + \rho d^{\pi^*}(s,a | x_1)
    =
    \doffline^{\pi^*}(s,a).
\end{align*}
Indeed, the expert data is incapable of distinguishing $\pi_0$ and $\pi^*$, since $d^{\pi_0}_{\widetilde{\Poffline}} = \doffline^{\pi^*}$, and $\Poffline$ is unknown. Unfortunately, as we've shown previously, $\pi_0$ achieves value zero. Notice that, unlike the previous section, one cannot distinguish $\pi^*$ from the catastrophic policy $\pi_0$ for \emph{any choice} of $\Ponline$.

\begin{algorithm}[t!]
\caption{{Confounded Imitation}}
\label{algo: partial imitation no shift}
\begin{algorithmic}[1]
% \small
\STATE {\bf input:} Expert data with missing context $\D^* \sim \doffline^{\pi^*}$, $\lambda > 0$, sensitivity bound $\delta \geq 0$.
\STATE {\bf init:} $\Upsilon = \emptyset$
% \STATE $q(\pi , \pi') := \max_{g: \s \times \A \mapsto \R}
% \expect*{s, a \sim \donline^{\pi}}{g(s,a)} - \expect*{s, a \sim \donline^{\pi'}}{f^*(g(s,a))}$
\FOR{$n = 1, \hdots$}
    \STATE Sample $u(s,a) \sim U[0, \delta], \forall s,a$
    \STATE $L^*(\pi; g_0) := \expect*{s, a \sim \donline^{\pi}(s,a)}{
         g_0(s,a)}
        -
        \expect*{s, a \sim \doffline^{\pi^*}(s,a) + u(s,a)}{ g_0(s,a)}$
    \STATE $L_i(\pi ; g_i) := \expect*{x \sim \Ponline, s, a \sim d^{\pi}(s, a | x)}{
         g_i(s,a,x)}
        -
        \expect*{x \sim \Ponline, s, a \sim d^{\pi_i}(s, a | x)}{ g_i(s,a,x)}\quad, i \geq 1$
    \STATE Compute $\pi_n$ by solving
    \begin{align}
        \min_{\pi \in \Pi_{\text{det}}}
        \max_{\abs{g_0} \leq \frac{1}{2}, \abs{g_i} \leq \frac{1}{2}}
        \brk[c]*{
        L^*(\pi; g_0(s,a)) 
        - 
        \lambda
        \min_i L_i(\pi; g_i(s,a,x))
        }
        \label{eq: confounded imitation}
    \end{align}
    \IF{$\pi_n \in \Upsilon$}
        \STATE Terminate and return 
        $
        \bar{\pi}(a|s,x) = 
        \frac{
            \sum_{i=1}^{n-1} d^{\pi_i}(s,a,x)
        }
        {
            \sum_{i=1}^{n-1} \sum_{a'} d^{\pi_i}(s,a',x)
        }
        $
    \ELSE
        \STATE $\Upsilon = \Upsilon \cup \brk[c]*{\pi_n}$
    \ENDIF
\ENDFOR
\end{algorithmic}
\end{algorithm}

\subsection{A Practical Algorithm} 

\Cref{algo: partial imitation no shift} describes our method for calculating the ambiguity set of \Cref{thm: ambiguity uniqueness}, and returns $\bar{\pi}$ of \Cref{thm: ambiguity policy selection}. At every iteration of the algorithm, we find a new policy in the set by minimizing the total variation distance (written in variational form) between $\donline^{\pi^*}(s,a)$ and $\donline^{\pi}(s,a)$, while regularizing it with the distance between $\pi$ and all previously collected $\pi_i \in \Upsilon$. \Cref{algo: partial imitation no shift} also uses a sensitivity parameter $\delta \geq 0$ (defined formally in \Cref{appendix: bounded confounding}) whenever bounded covariate shift is present. For this section we assume $\delta = 0$.

In practice, the functions $L^*$ and $L_i$ in lines~4 and~5 are estimated using samples from trajectories of $\pi, \pi_i$, and $\D^*$. We then solve the min-max problem of \Cref{eq: confounded imitation} using a parametric representations of $g_i$ and online gradient decent. The following proposition states that \Cref{algo: partial imitation no shift} indeed retrieves the set $\equivset{\pi^*}$.

\begin{restatable}{proposition}{algooneconvergence}
    Assume $\Poffline \equiv \Ponline$ and $\abs{\equivset{\pi^*}} < \infty$. Then there exists $\lambda^* > 0$ such that for any ${\lambda \in (0, \lambda^*)}$, \Cref{algo: partial imitation no shift} (with $\delta = 0$ sensitivity) will return $\bar{\pi}$ of \Cref{thm: ambiguity policy selection} after exactly $\abs{\equivset{\pi^*}}$ iterations.
\end{restatable}

\begin{figure}[t]
\centering
\includegraphics[width=0.49\linewidth]{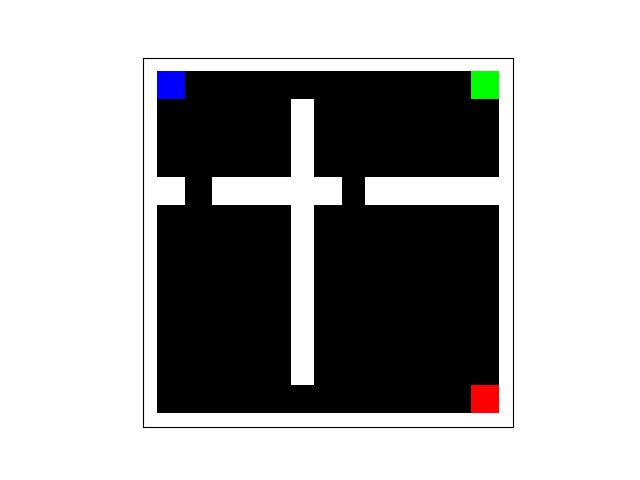}
\includegraphics[width=0.49\linewidth]{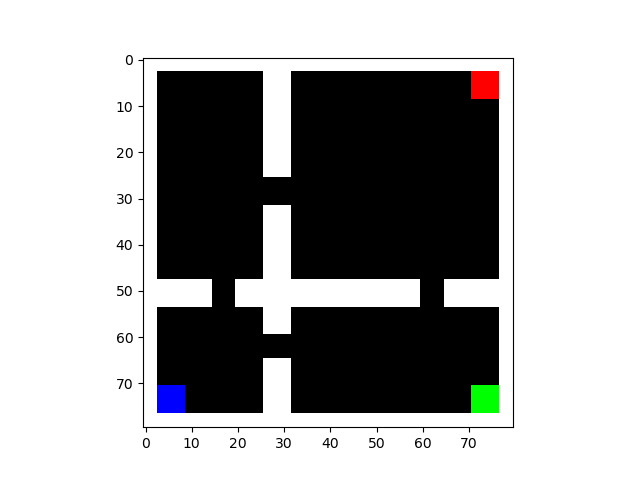}
\includegraphics[width=0.5\linewidth]{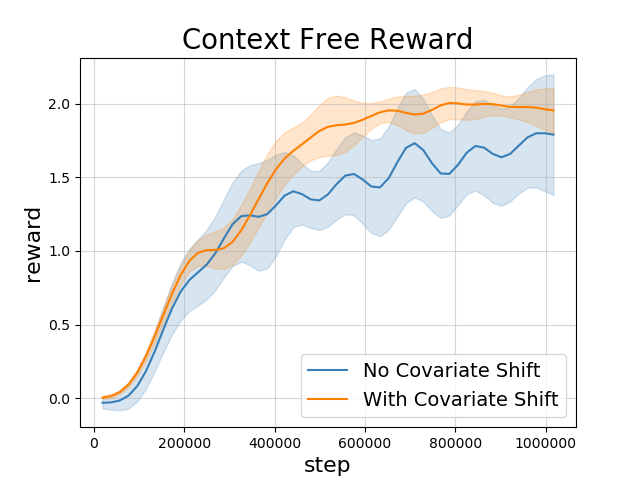}
\caption{Results for the rooms environment with covariate shift affecting only the distribution of walls. It is evident that whenever the reward is context-free comparable performance is obtained. Runs averaged over 5 seeds. }
\label{fig: rooms}
\end{figure}

\subsection{Imitation with Context-Free Reward}

We tested \Cref{algo: partial imitation no shift} on both the RecSim environment as well as a four-rooms environment with random instantiations of walls. Experiments for the RecSim environment are readily provided in \Cref{section: experiments}. Here we describe our simple four-rooms environment and show experiments w.r.t. \Cref{thm: context free reward}.

The four-rooms environment, as depicted in \Cref{fig: rooms}, is a $15\times 15$ grid-world in which an agent can take one of four actions: LEFT, RIGHT, UP, or DOWN. Each action moves the agent in the specified direction whenever no obstacle is present. The agent (shown in blue) must reach the (green) goal while avoiding the (red) mine. When the goal is reached the agent receives a reward of $+1$ and the episode terminates. In contrast, if the agent reaches the mine, she receives a reward of $-1$ and the episode terminates. The state space of the environment consists of the agent's $(\text{row}, \text{col})$ position in the world. The rest of the information in the environment is defined by the context $x$. Particularly, the context is defined by the position of the green goal, the position of the red mine, and the specific instantiation of walls (two instantiations are depicted in \Cref{fig: rooms}).

We trained an agent with full information (i.e., observed context, including goal location, mine location, and walls). We generated expert data w.r.t. the trained agent. To demonstrate the result of \Cref{thm: context free reward} we executed \Cref{algo: partial imitation no shift} with both a shifted distribution and the default distribution of walls. We did not change the distribution of goal and mine. Note that since the distribution of walls only affects the transition function and not the reward, we expect, by \Cref{thm: context free reward}, the optimal solution to remain the same. Indeed, as shown in \Cref{fig: rooms} after training an agent with no access to the contextual information of the walls in the expert data, the agent achieved comparable results both with and without covariate shift on the distribution of walls. 

This result seem surprising at first, as the walls are essential for solving the task at hand. Nevertheless, since the distribution of wall \emph{is observed} in the online environment, the partially observed expert data suffices to obtain an optimal policy. This settles with \Cref{thm: context free reward} which indeed states that this information is not needed in the expert data in order to obtain an optimal policy.

\section{Bounded Hidden Confounding}
\label{appendix: bounded confounding}

In this section we discuss the imitation learning problem under bounded hidden confounders. There are several ways to define boundness of unobserved confounders. In \Cref{section: imitation} we showed that, under \emph{arbitrary} covariate shift and context-free transitions, the imitation learning problem is impossible, i.e., one cannot rule out a catastrophic policy. We begin by considering the effect of bounded covariate shift, i.e., $\frac{\Ponline}{\Poffline} \leq C$. We then consider almost-context-free rewards, showing a tradeoff w.r.t. the hardness of the imitation problem.

\paragraph{A Sensitivity Perspective.}
A common approach in causal inference is to bound the bias of unobserved confounding through sensitivity analysis \citep{hsu2013calibrating,namkoong2020off,kallus2021minimax}. In our setting, this confounding bias occurs due to a covariate shift of the unobserved covariates. As we've shown in \Cref{thm: impossibility}, though these covariates are observed in the online environment, their shifted and unobserved distribution in the offline data can render catastrophic results. Therefore, we consider the odds-ratio bounds of the sensitivity in distribution between the online environment and the expert data, as stated formally below.
\begin{assumption}[Bounded Sensitivity]
\label{assumption: sensitivity}
    We assume that $\text{Supp}\brk*{\Poffline} \subseteq \text{Supp}\brk*{\Ponline}$ and that there exists some $\Gamma \geq 1$ such that for all $x \in \text{Supp}\brk*{\Poffline}$
    \begin{align*}
        \Gamma^{-1} 
        \leq \frac{\Ponline(x)(1-\Poffline(x))}{\Poffline(x)(1-\Ponline(x))}
        \leq
        \Gamma.
    \end{align*}
\end{assumption}

Next, we define the notion of $\delta$-ambiguity, a generalization of the ambiguity set in \Cref{def: ambiguity set}.
\begin{definition}[$\delta$-Ambiguity Set]
    For a policy $\pi \in \Pi$, we define the set of all deterministic policies that are $\delta$-close to $\pi$ by
    \begin{align*}
        \equivset{\pi}^{\delta} = \brk[c]*{\pi' \in \Pi_{\text{det}} : \abs{\donline^{\pi'}(s, a) - \doffline^{\pi}(s, a)} < \delta, s \in \s, a \in \A}.
    \end{align*}
\end{definition}
Similar to \Cref{def: ambiguity set}, the $\delta$-ambiguity set considers all deterministic policies with a marginalized stationary distribution of distance at most $\delta$ from $\pi$. The following results shows that $\equivset{\pi^*}^{\Gamma - 1}$ is a sufficient set of candidate optimal policies, as long as \Cref{assumption: sensitivity} holds for some $\Gamma \geq 1$.

\begin{restatable}{theorem}{deltasufficiency}[Sufficiency of $\equivset{\pi^*}^{\Gamma - 1}$]
    Let \Cref{assumption: sensitivity} hold for some $\Gamma \geq 1$. Then $\pi^* \in \equivset{\pi^*}^{\Gamma - 1}$.
\end{restatable}

The above result suggests that \Cref{algo: partial imitation no shift} can be executed over $\equivset{\pi^*}^{\Gamma - 1}$ by adding $\delta=\Gamma-1$ additive uniform noise to $\doffline^{\pi^*}(s,a)$ (see Line 4 of \Cref{algo: partial imitation no shift}), and executing the algorithm for a finite number of iterations, finally selecting a robust policy from the approximate set. 

\paragraph{Context Reconstruction.} When bounded covariate shift is present, one might attempt to learn an inverse mapping of contexts from observed trajectories in the data.

We denote by $P_\rho^\pi$ the probability measure over contexts $x \in \X$ and trajectories ${\tau = (s_0, a_0, s_1, a_1, \hdots s_H)}$ as induced by the policy $\pi$ and context distribution $\rho$. That is,
\begin{align*}
    P_\rho^\pi(x, \tau)
    =
    \rho(x)
    \nu(s_0|x)
    \prod_{t=0}^{H-1} 
    P(s_{t+1} | s_t, a_t, x)
    \pi(a_t | s_t, x).
\end{align*}
As the true context is observed in the online environment, we can calculate for any $\pi$ the quantity $P_{\Ponline}^\pi(x, \tau)$. As the expert data distribution was generated by the marginalized distribution ${P_{\Ponline}^{\pi^*}(\tau) = \sum_{x \in \X} P_{\Ponline}^{\pi^*}(x, \tau)}$, it is unclear if knowledge of $P_{\Ponline}^\pi(x, \tau)$ is beneficial.

Fortunately, whenever \Cref{assumption: sensitivity} holds, a high probability of reconstructing a context in the online environment induces a high probability of reconstructing it in the expert data. To see this, assume that there exists $\delta \in [0,1]$ such that for all ${\pi \in \equivset{\pi^*}^{\Gamma - 1}}$, $\tau \in \text{Supp}(P^\pi_{\Poffline}(\tau))$, there exists $x \in \X$ such that
\begin{align}
\label{assumption: delta identifiability}
    P^\pi_{\Ponline}(x | \tau) \geq \min\brk[c]*{(1-\delta)\brk*{\Ponline(x) + \Gamma\brk*{1-\Ponline(x)}}, 1}.
\end{align}
That is, we assume that for any policy that $\delta$-ambiguous to $\pi^*$, and any induced trajectory of $x \in \X$, one can with high probability identify $x$ in the online environment. Importantly, this property can be verified in the online environment. When Assumption~\ref{assumption: sensitivity} and~\ref{assumption: delta identifiability} hold, we get that
\begin{align*}
    P^\pi_{\Poffline}(x | \tau)
    =
    \frac{P^\pi_{\Poffline}(\tau|x)\Poffline(x)}{P^\pi_{\Poffline}(\tau)} 
    \geq
    \frac{P^\pi_{\Poffline}(\tau|x)}{P^\pi_{\Poffline}(\tau)}
    \frac{\Ponline(x)}{\Ponline(x) + \Gamma\brk*{1-\Ponline(x)}} 
    =
    \frac{P^\pi_{\Ponline}(x | \tau)}{\Ponline(x) + \Gamma\brk*{1-\Ponline(x)}} 
    \geq
    1 - \delta.
\end{align*}
In other words, we can reconstruct $x$ with probability $1-\delta$ for any trajectory $\tau$ which satisfies the above. This allows us to deconfound essential parts of the expert data, rendering it useful for the imitation problem, even when reward is not provided. We leave further analysis of this direction for future work.

\paragraph{Context-Dependent Reward.}

In \Cref{thm: context free reward} we showed that whenever the reward is independent of the context then the imitation problem is easy, in the sense that any policy $\pi_0 \in \equivset{\pi^*}$ is also an optimal policy. Here, we relax the assumption on the reward, and instead assume bounded dependence of the reward on the context. The following definition upper bounds the confounding effect of the reward w.r.t. the context.
\begin{definition}
\label{definition: reward dependence}
    Let $\epsilon: \X \mapsto \R$ such that
    \begin{align*}
        \min_{r_0: \s \times \A \mapsto \R} \abs{r(s,a,x) - r_0(s,a)} \leq \epsilon(x) \quad,s \in \s, a \in \A, x \in \X
    \end{align*}
\end{definition}
Using the above definition, we can now show that any policy in $\equivset{\pi^*}$ is still approximately optimal, as shown by the following result.
\begin{restatable}{theorem}{contextdependentreward}[Context Dependent Reward]
\label{thm: context dependent reward}
Let $\epsilon: \X \mapsto \R$ of \Cref{definition: reward dependence}. Denote ${\epsilon_{oe} = \expect*{x \sim \Ponline(x)}{\epsilon(x)} + \expect*{x \sim \Poffline(x)}{\epsilon(x)}}$. Then for any $\pi^* \in \Pi^*_\M$, $\pi_0 \in \equivset{\pi^*}$
\begin{align*}
    v(\pi_0)
    \geq
    v(\pi^*) 
    -
    \epsilon_{oe}.
\end{align*}
\end{restatable}
A direct corollary for the above result states that for $\epsilon: \X \mapsto \R$ of \Cref{definition: reward dependence}, if $\epsilon(x) = \epsilon$, for all $x \in \X$, then for $\pi_0, \pi^*$ of \Cref{thm: context dependent reward}, it holds that $v(\pi_0) \geq v(\pi^*) - 2\epsilon$. That is, $\pi_0$ is an approximately optimal policy.

\section{Relation to Causal Inference}
\label{appendix: relation to causal inference}

\begin{figure}[t]
\centering
\includegraphics[width=0.5\linewidth]{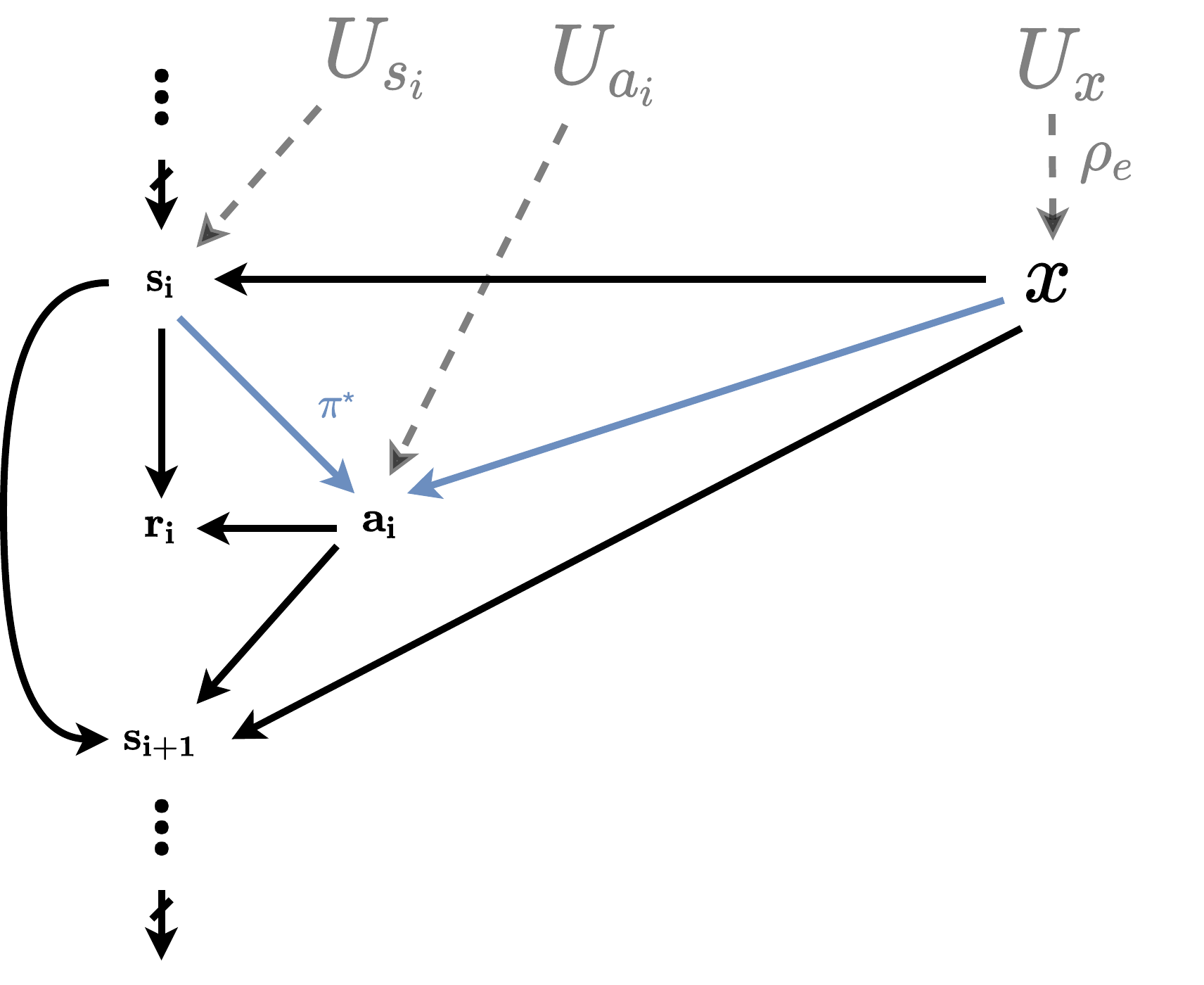}
\caption{\textbf{Contextual MDP Causal Diagram}.}
\label{fig: causal diagram}
\end{figure}

Our work is focused on the problem of hidden confounders in expert data for imitation and reinforcement learning. We have chosen to write the paper in terminology famililar to the RL community. In this section we address and formalize the problem in Causal Inference (CI) terminology. We begin by defining Structural Causal Models (SCM, \citet{Pearl:2009:CMR:1642718}) -- a basic building block of our framework. We then show how the confounded imitation problem can be formalized as an intervention over a specific SCM. Generally speaking, the causal view casts the environment, namely the expert environment generating the offline data, and the online environment, as confounders. 

\begin{definition}[Structural Causal Models]
    A Structural Causal Model (SCM) is a tuple ${\M = (U, V, \mathcal{F}, P(U))}$ where $U$ is a set of exogenous variables and $V$ is a set of endogenous variables. $\mathcal{F}$ is a set of functions such that $f_i \in \mathcal{F}$ are functions mapping a set of endogenous variables $Pa_i \subseteq V \backslash \brk[c]*{V_i}$ and a set of exogenous variables $U_i \subseteq U$ to the domain of $V_i$, i.e., $V_i = f_i(Pa_i, U_i)$. Finally, $P(u)$ is a probability distribution over the set of exogenous variables $U$.  We assume that the SCM is recursive, i.e., that the causal diagram associated with it is acyclic.
\end{definition}

Every SCM $\M$ is associated with a causal diagram $\G$, as depicted in \Cref{fig: causal diagram}. Our framework relies largely on the formulation of stochastic interventions, as proposed in \citet{correa2020calculus}. We consider stochastic, conditional (non-atomic) interventions, defined by regime indicators $\sigma_Z$ \citep{pearl2000causality,correa2020calculus}, defined formally below.

\begin{definition}[Non-Atomic Interventions]
    Given a SCM $\M = (U, V, \mathcal{F}, P(U))$ and a subset ${Z \subseteq V}$, an intervention ${\sigma_Z = \brk[c]*{\sigma_{Z_1}, \hdots, \sigma_{Z_n}}}$ defines a new SCM $\M_{\sigma_Z} =  (U, V, \mathcal{F}^*, P(U))$ in which the set of functions $\mathcal{F}$ is changed to ${\mathcal{F}^* = \brk[c]*{f^*_i}_{i : V_i \in \brk[c]*{Z_j}_{j=1}^n} \bigcup \brk[c]*{f_i}_{i : V_i \in V \backslash \brk[c]*{Z_j}_{j=1}^n}}$.
\end{definition}

Non-atomic interventions are a generalization of the classic atomic $\textbf{do}(X=x)$ interventions, defined by the SCM $\M_z$ and causal diagram $\G_{\overline{Z}}$ in which all edges incoming into $Z$ are removed. We have that
\begin{align*}
    P(y | \textbf{do}(Z=z)) = P(y | z ; \sigma_Z = \textbf{do}(Z=z)).
\end{align*}
Atomic interventions replace function in $\F$ by constant functions, whereas non-atomic interventions use general functions. For notational simplicity, when a single intervention is applied to some $f_i \in \F$, we denote it by $\sigma_Z = \textbf{d}(f_i \gets f_i^*)$, indicating that in the interventional distribution $f_i^*$ is used instead of $f_i$. Next we define the identifiability of a causal effect under an intervention, as follows.
\begin{definition}[Identifiability]
    Let $X, Y, Z \subseteq V$ with $Y \cap Z = \emptyset$ in some SCM with causal diagram $\G$. Given an intervention $\sigma_Z = \brk[c]*{\sigma_{Z_1}, \hdots, \sigma_{Z_n}}$, the causal effect $P(y | x, \sigma_Z)$ is said to be identifiable from $V' \subseteq V$ if it can be uniquely computed from $P(V')$ for every assignment $(y,x)$ in every model that induces $\G$ and $P(V')$.
\end{definition}
Following the model definitions of \Cref{section: preliminaries}, we define the contextual MDP SCM as follows.
\begin{definition}
A contextual MDP SCM is defined by the causal diagram of \Cref{fig: causal diagram}. For some horizon $H > 0$, the SCM is defined by the set of endogenous variables ${V = \brk[c]*{s_i}_{i=0}^H \cup \brk[c]*{a_i}_{i=0}^H \cup \brk[c]*{x} \cup \brk[c]*{r_i}_{i=0}^H}$ (denoting the states, actions, context and rewards, respectively), a set of exogenous variables $U$, and functions $\mathcal{F}=\brk[c]*{f_{s_i},f_{a_i},f_{r_i}, f_{\Poffline}, f_{\nu_0}}$, where $f_{s_i}$ correspond to the transition function, $f_{a_i}$ the expert policy, $f_{r_i}$ the reward function, $f_{\Poffline}$ the context expert distribution, and $f_{\nu_0}$ the initial context-dependent state distribution.
\end{definition}
Relating to our formal definition of our model in \Cref{section: preliminaries}, with slight abuse of notations, the functions $\brk[c]*{f_s,f_g,f_a, f_{\Poffline}, f_{\nu_0}}$ adhere to the following relations
\begin{align*}
    &P(s_{i+1}=s'|s_i=s,a_i=a,x) = P(f_{s_i}(s,a,x,U)) \\
    &P(r_i=r | s_i=s,a_i=a,x) = \delta(f_{r_i}(s,a,x) = r) \\
    &\pi^*(a_i=a|s_i=s,x) = P(f_{a_i}(a,s,x,U) \\
    &\rho_e(x) = P(f_{\Poffline}(x,U)) \\
    &\nu_0(s_0 | x) = P(f_{\nu_0}(s_0, x, U)),
\end{align*}
where $\delta(\cdot)$ indicates the Dirac delta distribution. 

We are now ready to define the confounded imitation problem. We define the (non-atomic) intervention $\sigma_x = \textbf{do}\brk*{f_{\Poffline} \gets f_{\Ponline}}$ which replaces $f_{\Poffline}$ with $f_{\Ponline}$ in the contextual MDP SCM defined above. The goal of imitation learning is then to identify the quantities
\begin{align}
\label{eq: identifiability}
    P(a_i | s_i, x, \sigma_x = \textbf{do}\brk*{f_{\Poffline} \gets f_{\Ponline}}) \quad, 0 \leq i \leq H-1,
\end{align}
where, importantly, we assume we \emph{only} have access to $P(s_i,a_i)$, $P(s_{i+1}|s_i,a_i,x)$, $P(s_0 | x)$, and ${P(x | \sigma_x = \textbf{do}\brk*{f_{\Poffline} \gets f_{\Ponline}})}$. Notice that in our setting, $P(s_{i+1}|s_i,a_i,x)$, $P(s_0 | x)$, and ${P(x | \sigma_x = \textbf{do}\brk*{f_{\Poffline} \gets f_{\Ponline}})}$ correspond to known quantities of the online environment, whereas $P(s_i,a_i)$ corresponds to the (partially observed) offline expert data. We also emphasize that $P(s_{i+1}|s_i,a_i,x)$ is not dependent on the intervention $\sigma_Z$. That is,
\begin{align*}
    P(s_{i+1}|s_i,a_i,x, \sigma_x = \textbf{do}\brk*{f_{\Poffline} \gets f_{\Ponline}}) = P(s_{i+1}|s_i,a_i,x).
\end{align*}

\begin{remark*}
    Our work studies a slightly different version of the identifiability problem in \Cref{eq: identifiability}, as we only wish to identify an optimal policy from the set $\Pi^*_\M$, as opposed to the single specific policy $\pi^*$. This requirement can be formalized by defining an extended SCM which includes all optimal policies in $\Pi^*_\M$, with the assumption that only one is observed (corresponding to the expert data).
\end{remark*}

\begin{algorithm}[t!]
\caption{{RL using Expert Data with Unobserved Confounders (Complete Algorithm)}}
\label{algo: complete ogd}
% \label{algo: partial imitation no shift}
\begin{algorithmic}[1]
% \small
\STATE {\bf input:} Expert data with missing context $\D^*$, ${\lambda,\alpha, B, N, M > 0}$, policy optimization algorithm \texttt{ALG-RL}
\STATE {\bf init:} Policy $\pi^0$, global bonus reward network $g^*_\theta$ 
\FOR{$k = 1, \hdots$}
    \STATE Generate dataset of rollouts $\mathcal{R}_k \sim \donline^{\pi_{k-1}}(s,a)$
    \STATE Initialize local networks $g^m_{\theta_m} \gets g_\theta, m \in [M]$
    \FOR{$m = 1, \hdots M$}
        \STATE Sample weight vector $w_m$ uniformly from $\Delta_n$
        \FOR{$e =1 \hdots N$}
            \STATE Sample batch uniformly from $\mathcal{R}_k$, i.e., $\brk[c]*{s_i, a_i}_{i=1}^B \overset{U}{\sim} \mathcal{R}_k$
            \STATE Sample batch according to weights $w_m$ from $\D^*$, i.e., $\brk[c]*{s_i^e, a_i^e}_{i=1}^B \overset{w_m}{\sim} \D^*$ 
            \STATE Update $g^m_{\theta_m}$ according to 
            \begin{align*}
                \nabla_{\theta_m} L_m(\theta_m) 
                =
                \frac{1}{B}\sum_{i=1}^B\nabla_{\theta_m} \brk[s]*{
                (1-\alpha)f^*(g^m_{\theta_m}(s_i^e, a_i^e)) 
                + \alpha f^*(g^m_{\theta_m}(s_i, a_i)) 
                - g^m_{\theta_m}(s_i, a_i)}
            \end{align*}
        \ENDFOR
    \STATE $m^* \in \arg\min_{m \in [M]} L_m(\theta_m)$
    \STATE Update global parameters from the selected local network $g^*_\theta \gets g^{m^*}_{\theta_{m^*}}$
    \STATE $\pi^k \gets \text{\texttt{ALG-RL}}(r(s,a,x) - \lambda g^*_\theta(s,a))$
    \ENDFOR
\ENDFOR
\end{algorithmic}
\end{algorithm}

\section{Implementation Details}
\label{appendix: implementation details}

Our experiments were based off of the recently proposed assistive-gym \citep{erickson2020assistive} and recsim \citep{ie2019recsim} environments. In this section we discuss further implementation details, hyperparameters, context distributions, and generation of the expert data.

\begin{table*}[t!]
  \small\centering
  \begin{tabular}{lcc}
    \toprule
    {\bfseries Name} & {\bfseries Value}  & {\bfseries Comments}\\
    \midrule\midrule[.1em]
    Batch size & $128$ &  \\
    \midrule[.1em] 
    Learning rate & $\num{5e-5}$ &  \\
    \midrule[.1em] 
    Rollout size & $19,200$ &  \\
    \midrule[.1em] 
    Total timesteps & $\num{5e6}$ &  \\
    \midrule[.1em] 
    Num epochs & 50 & How many training epochs to do after each rollout \\
    \midrule[.1em] 
    $\gamma$ & $0.95$ & Discount factor \\
    \midrule[.1em] 
    kl coef & $0.2$ & Initial coefficient for KL divergence \\
    \midrule[.1em] 
    kl target & $0.01$ & Target value for KL divergence \\
    \midrule[.1em] 
    GAE $\lambda$ & $1$ & The GAE (lambda) parameter \\
    \midrule[.1em] 
    Num workers & $40$ & \\
    \midrule[.1em] 
    \bottomrule
  \end{tabular}
  \caption{Hyper-parameters used to train the PPO agent. \label{table:ppo-params}}
\end{table*}

\begin{table*}[t!]
  \small\centering
  \begin{tabular}{lcc}
    \toprule
    {\bfseries Name} & {\bfseries Value}  & {\bfseries Comments}\\
    \midrule\midrule[.1em]
    Batch size & $128$ &  \\
    \midrule[.1em] 
    Learning rate & $\num{1e-4}$ &  \\
    \midrule[.1em] 
    Imitation Method & $\chi$-divergence &  \\
    \midrule[.1em] 
    Num epochs & 50 & How many training epochs to do after each rollout \\
    \midrule[.1em] 
    $\alpha$ & $0.9$ &  $D_f$ regularization coefficient \\
    \midrule[.1em] 
    $M$ & $10000$ & Budget for CTS optimizer \\
    \midrule[.1em] 
    \bottomrule
  \end{tabular}
  \caption{Hyper-parameters used for imitation and CTS. \label{table:cts-params}}
\end{table*}

\paragraph{Algorithm Details.} A complete description of \Cref{algo: ogd} is presented in \Cref{algo: complete ogd}. Specific hyperparameters used are shown in \Cref{table:ppo-params,table:cts-params}. We implemented the algorithm using the RLlib framework \citep{liang2018rllib}. We used PPO \citep{schulman2017proximal} as our policy-optimization algorithm. All neural networks consisted of two-layer fully connected MLPs with 100 parameters in each layer. We used the same rollout buffer (of size 19200 samples) for both our PPO agent as well as our imitation module, which estimated the augmented reward. 

Motivated by \citet{kostrikov2019imitation}, we regularized the expert demonstrations with samples from $d^\pi$. Particularly, we let $\alpha \in (0,1]$, such that $1-\alpha$ corresponds to the probabiilty of sampling an expert example and $\alpha$ corresponds to the probability of sampling from the replay. This leads to minimizing an augmented version of the $f$-divergence which can be written as
\begin{align*}
    \min_{g: \s \times \A \mapsto \R} 
     \expect*{s,a \sim \donline^\pi(s,a | x)}{g(s,a) - \alpha f^*(g(s,a)) } 
     - 
     (1-\alpha)\expect*{s,a \sim \doffline^{\pi^*}(s,a)}{f^*(g(s,a))}.
\end{align*}

Our imitation module consisted of two networks $g_\theta$ and $h_\theta$ as proposed in \citet{fu2018learning}. The ``done" signal was also added to the state for training the imitation module. For training CTS we used the Nevergrad optimization platform \citep{nevergrad} with a budget of 10000 and one worker. Here a copied version of the networks $g_\theta$ and $h_\theta$ were used to for initialization and then approximate the minimum $D_f$.

For choosing $\lambda$ we used an adaptive strategy which ensured $\lambda$ balanced the RL objective with the imitation objective. Specifically, we used the following tradeoff between reward $r$ and bonus $g$
\begin{align*}
    (1-\lambda_{\text{adap}})r(s,a) + \lambda_{\text{adap}} g(s,a),
\end{align*}
where $\lambda_{\text{adap}} = \frac{r_{\text{mean}}}{r_{\text{mean}} + g_{\text{mean}}}$. Here $r_{\text{mean}}$ corresponds to the average reward in the replay buffer and $g_{\text{mean}}$ to the average bonus in the replay buffer. By averaging the two, we maintained a similar scale to effectively use the expert data in all the evaluated environments without optimizing for $\lambda$.

\paragraph{Context Distribution.}
For each environment we used a varying context distribution in the expert data, with increasing distance to that of the online environment. The context distribution for the RecSim environment is formally described in \Cref{section: experiments}. For the assistive-gym environment the context was defined by the following features: gender, mass, radius, height, patient impairment, and patient preferences. The patient's mass, radius, and height distributions were dependent on gender. The patient's impairment was given by either limited movement, weakness, or tremor (with sporadic movement). Finally, the patient's preferences were affected by the velocity and pressure of touch forces applied by the robot. We used default average values that were provided with the simulator. Particularly, we used the following distributions for each feature
\begin{align*}
    &\text{gender} \sim \text{Bern}(p_{\text{male}}) \\
    &\text{mass}(\text{gender}) 
    \sim \mathcal{N}(\mu_{\text{mass}}(\text{gender}) , \sigma^2_{\text{mass}}) \\
    & \text{radius}(\text{gender}) 
    \sim \mathcal{N}(\mu_{\text{radius}}(\text{gender}) , \sigma^2_{\text{radius}}) \\
    & \text{height}(\text{gender}) 
    \sim \mathcal{N}(\mu_{\text{height}}(\text{gender}) , \sigma^2_{\text{height}}) \\
    & \text{velocity weight}
    \sim \text{Unif}([\ell_{\text{vel}}, u_{\text{vel}}]) \\
    & \text{force nontarget weight}
    \sim \text{Unif}([\ell_{\text{target}}, u_{\text{target}}]) \\
    & \text{high forces}
    \sim \text{Unif}([\ell_{\text{high forces}}, u_{\text{high forces}}]) \\
    & \text{food hit weight}
    \sim \text{Unif}([\ell_{\text{hit}}, u_{\text{hit}}]) \\
    & \text{food velocity weight}
    \sim \text{Unif}([\ell_{\text{food vel}}, u_{\text{food vel}}]) \\
    & \text{high pressures weight}
    \sim \text{Unif}([\ell_{\text{high pressure}}, u_{\text{high pressure}}]) \\
    & \text{impairment} \sim \text{Multinomial}(p_{\text{none}}, p_{\text{limits}}, p_{\text{weakness}}, p_{\text{tremor}}).
\end{align*}
The values for each distribution are provided in \Cref{table: assistive params}. For setting covariate shift, we used a a set of distributions that were shifted w.r.t. the default context distribution. We then sampled a shifted distribution w.p. $\beta$ and the default distribution w.p. $1-\beta$. That is, when $\beta = 1$, the user sampled a context only from the shifted distribution. \Cref{table: shifted assistive params} shows an example of one of the shifted distribution that were used.

\begin{table*}[t!]
  \small\centering
  \begin{tabular}{|lc|lc|lc|}
    \toprule
    {\bfseries Name} & {\bfseries Value} & {\bfseries Name} & {\bfseries Value} & {\bfseries Name} & {\bfseries Value} \\
    \midrule\midrule[.1em]
    $p_{\text{male}}$ & $0.3$ & $\ell_{\text{vel}}$ & $0.225$ & $\ell_{\text{high pressure}}$ & $0.009$\\
    \midrule[.1em] 
    $\mu_{\text{mass}}(\text{male})$ & $78.4$  & $u_{\text{vel}}$ & $0.275$  & $u_{\text{high pressure}}$ & $0.011$ \\
    \midrule[.1em] 
    $\mu_{\text{mass}}(\text{female})$ & $62.5$ & $\ell_{\text{target}}$ & $0.009$ & $p_{\text{none}}$ & $0.1$\\
    \midrule[.1em] 
    $\sigma^2_{\text{mass}}$ & $10$ & $u_{\text{target}}$ & $0.011$ & $p_{\text{limits}}$ & $0.4$\\
    \midrule[.1em] 
    $\mu_{\text{radius}}(\text{male})$ & $1$ & $\ell_{\text{high forces}}$ & $0.045$ & $p_{\text{weakness}}$ & $0.3$\\
    \midrule[.1em] 
    $\mu_{\text{radius}}(\text{female})$ & $1$ & $u_{\text{high forces}}$ & $0.055$ & $p_{\text{tremor}}$ & $0.2$ \\
    \midrule[.1em] 
    $\sigma^2_{\text{radius}}$ & $0.1$ & $\ell_{\text{hit}}$ & $0.9$ &&\\
    \midrule[.1em] 
    $\mu_{\text{height}}(\text{male})$ & $1$ & $u_{\text{hit}}$ & $1.1$ &&\\
    \midrule[.1em] 
    $\mu_{\text{height}}(\text{female})$ & $1$ & $\ell_{\text{food vel}}$ & $0.9$ && \\
    \midrule[.1em] 
    $\sigma^2_{\text{height}}$ & $0.1$ & $\ell_{\text{food vel}}$ & $1.1$ && \\
    \midrule[.1em] 
    \bottomrule
  \end{tabular}
  \caption{Parameters for context distribution used in assistive-gym \label{table: assistive params}}
\end{table*}

\begin{table*}[t!]
  \small\centering
  \begin{tabular}{|lc|lc|lc|}
    \toprule
    {\bfseries Name} & {\bfseries Value} & {\bfseries Name} & {\bfseries Value} & {\bfseries Name} & {\bfseries Value} \\
    \midrule\midrule[.1em]
    $p_{\text{male}}$ & $0.8$ & $\ell_{\text{vel}}$ & $0.225$ & $\ell_{\text{high pressure}}$ & $0.009$\\
    \midrule[.1em] 
    $\mu_{\text{mass}}(\text{male})$ & $8 8.4$  & $u_{\text{vel}}$ & $0.275$  & $u_{\text{high pressure}}$ & $0.111$ \\
    \midrule[.1em] 
    $\mu_{\text{mass}}(\text{female})$ & $72.5$ & $\ell_{\text{target}}$ & $0.007$ & $p_{\text{none}}$ & $0.1$\\
    \midrule[.1em] 
    $\sigma^2_{\text{mass}}$ & $20$ & $u_{\text{target}}$ & $0.016$ & $p_{\text{limits}}$ & $0.1$\\
    \midrule[.1em] 
    $\mu_{\text{radius}}(\text{male})$ & $0.9$ & $\ell_{\text{high forces}}$ & $0.035$ & $p_{\text{weakness}}$ & $0.1$\\
    \midrule[.1em] 
    $\mu_{\text{radius}}(\text{female})$ & $0.9$ & $u_{\text{high forces}}$ & $0.06$ & $p_{\text{tremor}}$ & $0.7$ \\
    \midrule[.1em] 
    $\sigma^2_{\text{radius}}$ & $0.2$ & $\ell_{\text{hit}}$ & $0.4$ &&\\
    \midrule[.1em] 
    $\mu_{\text{height}}(\text{male})$ & $1.1$ & $u_{\text{hit}}$ & $2.1$ &&\\
    \midrule[.1em] 
    $\mu_{\text{height}}(\text{female})$ & $1.1$ & $\ell_{\text{food vel}}$ & $0.4$ && \\
    \midrule[.1em] 
    $\sigma^2_{\text{height}}$ & $0.2$ & $\ell_{\text{food vel}}$ & $2.1$ && \\
    \midrule[.1em] 
    \bottomrule
  \end{tabular}
  \caption{Parameters for one of the shifted context distribution used in assistive-gym \label{table: shifted assistive params}}
\end{table*}

\paragraph{Expert Data Generation.}
For the assistive-gym experiments we a dense reward function for generating the expert data and a sparse one for our experiments using the expert data. Specifically, the dense reward function used the environment's default reward function, defined by
\begin{align*}
     w_1 \cdot \text{distance to goal} + w_2 \cdot \text{action} + w_3 \cdot \text{task specific reward} + w_4 \cdot \text{preference score},
\end{align*}
where the preferences were weighted according to the context features. Specific weights are provided in the implementation of assistive-gym \citep{erickson2020assistive}. The sparse reward function did not use the distance to goal (i.e., $w_1 = 0$).

\newpage
\section{Missing Proofs}
\label{appendix: missing proofs}

\subsection{Proofs for \Cref{section: imitation}}

We begin by proving two auxilary lemmas.
\begin{lemma}
\label{lemma: Pi0 equivalence}
    Let $\pi_2 \in \equivset{\pi_1}$. Then, $\equivset{\pi_1} = \equivset{\pi_2}$.
\end{lemma}
\begin{proof}
    We show that $\equivset{\pi_1} \subseteq \equivset{\pi_2}$ and $\equivset{\pi_2} \subseteq \equivset{\pi_1}$.
    
    Let $\pi \in \equivset{\pi_1}$, then
    $
        d^\pi(s,a) = d^{\pi_1}(s,a).
    $
    By our assumption, $\pi_2 \in \equivset{\pi_1}$, then
    $
        d^{\pi_2}(s,a) = d^{\pi_1}(s,a).
    $
    Hence,
    $
        {d^\pi(s,a) = d^{\pi_2}(s,a).}
    $
    That is, $\pi \in \equivset{\pi_2}$. This proves $\equivset{\pi_1} \subseteq \equivset{\pi_2}$.
    
    Similarly, let $\pi \in \equivset{\pi_2}$, then
    $
        d^\pi(s,a) = d^{\pi_2}(s,a).
    $
    By our assumption, $\pi_2 \in \equivset{\pi_1}$, then
    $
        d^{\pi_2}(s,a) = d^{\pi_1}(s,a).
    $
    Hence,
    $
        {d^\pi(s,a) = d^{\pi_1}(s,a)}.
    $
    That is, $\pi \in \equivset{\pi_1}$. This proves $\equivset{\pi_2} \subseteq \equivset{\pi_1}$, completing the proof.
\end{proof}

\begin{lemma}
\label{lemma: unique optimal policy}
    Let $\pi_0$ be a \underline{deterministic} policy and let $\M_0 = (\s, \A, \X, P, r_0, \gamma)$ such that ${r_0(s,a,x) = \mathbf{1}\brk[c]*{a = \pi_0(s,x)}}$. Then $\pi_0$ is the \underline{unique, optimal} policy in $\M_0$.
\end{lemma}
\begin{proof}
    By definition of $\pi_0$ and $r_0$,
    \begin{align*}
        r_0(s,\pi_0(s,x),x) = 1, \forall s \in \s, x \in \X.
    \end{align*}
    In particular, $\expect*{\pi_0}{r_0(s_t,a_t,x)} = 1$. Then
    \begin{align*}
        V^*_{\M_0} \leq (1-\gamma)\sum_{t=0}^\infty \gamma^t = \expect*{\pi_0}{(1-\gamma)\sum_{t=0}^\infty \gamma^t r_0(s_t,a_t,x)} = V^{\pi_0}_{\M_0}.
    \end{align*}
    This proves $\pi_0$ is an optimal policy. To prove uniqueness, assume by contradiction there exists an optimal policy $\pi_1 \neq \pi_0$. Then, 
    \begin{align*}
        V^{\pi_1}
        =
        \expect*{s,a,x \sim d^{\pi_1}(s,a,x)}{\mathbf{1}\brk[c]*{a = \pi_0(s,x)}} 
        =
        \expect*{s,x \sim d^{\pi_1}(s,x)}{\expect*{a \sim \pi_1(\cdot | s, x)}{\mathbf{1}\brk[c]*{a = \pi_0(s,x)}}} 
        <
        1
        =
        V^{\pi_0}_{\M_0}.
    \end{align*}
    In contradiction to $\pi_1$ is optimal. Then, $\pi_0$ is a unique optimal policy.
\end{proof}

We are now ready to prove \Cref{thm: ambiguity uniqueness}.
\ambiguitythm*
\begin{proof}
    Let $\pi^* \in \Pi^*_{\M}$ and let $\pi_0 \in \equivset{\pi^*}$. By \Cref{lemma: Pi0 equivalence}, as $\pi_0 \in \equivset{\pi^*}$, it holds that $\equivset{\pi^*} = \equivset{\pi_0}$. Next, choosing $r_0(s,a,x) = \mathbf{1}\brk[c]*{a = \pi_0(s,x)}$, by \Cref{lemma: unique optimal policy} we get that $\pi_0$ is an optimal policy in $\M_0$. This proves $\pi_0 \in \Pi^*_{\M_0}$. Finally, by \Cref{lemma: unique optimal policy}, $\Pi^*_{\M_0} = \brk[c]*{\pi_0}$, proving $\pi^* \notin \Pi^*_{\M_0}$ if and only if $\pi^* \neq \pi_0$.
\end{proof}

\ambiguityselectthm*
\begin{proof}
    Let $\tilde{\pi}$ as defined. Then by linearity of expectation
    \begin{align*}
        V^{\tilde{\pi}}_{\M} 
        = 
        \expect*{s,a,x \sim d^{\tilde{\pi}}}{r(s,a,x)}
        =
        \frac{1}{\abs{\equivset{\pi^*}}}
        \sum_{\pi \in \equivset{\pi^*}} 
        \expect*{s,a,x \sim d^{\pi}}{r(s,a,x)}
        =
        \frac{1}{\abs{\equivset{\pi^*}}}
        \sum_{\pi \in \equivset{\pi^*}} 
        V^{\pi}_{\M}. 
    \end{align*}
    Denote $B^* = \Pi^*_\M \cap \equivset{\pi^*}$, then
    \begin{align*}
        V^{\tilde{\pi}}_{\M}
        &=
        \frac{1}{\abs{\equivset{\pi^*}}} \sum_{\pi \in B^*} V^{\pi}_{\M}
        +
        \frac{1}{\abs{\equivset{\pi^*}}} \sum_{\equivset{\pi^*} \backslash B^*} V^{\pi}_{\M} \\
        &=
        \frac{\abs{B^*}}{\abs{\equivset{\pi^*}}} 
        V^*_\M 
        + 
        \frac{1}{\abs{\equivset{\pi^*}}} \sum_{\equivset{\pi^*} \backslash B^*} V^{\pi}_{\M} \\
        &\geq
        \frac{\abs{B^*}}{\abs{\equivset{\pi^*}}} 
        V^*_\M 
        + 
        \frac{\abs{\equivset{\pi^*} \backslash B^*}}{\abs{\equivset{\pi^*}}} \min_{\pi \in \equivset{\pi^*} \backslash B^*} V^{\pi}_{\M} \\
        &\geq
        \frac{\abs{B^*}}{\abs{\equivset{\pi^*}}} 
        V^*_\M 
        + 
        \frac{\abs{\equivset{\pi^*} \backslash B^*}}{\abs{\equivset{\pi^*}}} \min_{\pi \in \equivset{\pi^*}} V^{\pi}_{\M},
    \end{align*}
    completing the proof.
\end{proof}

\impossibilitythm*
\begin{proof}
    We first sketch the proof for the special case ${\X = \{x_0, x_1\}}$, $\A = \{a_0, a_1\}$ and a singleton state space $\s = \brk[c]*{s_0}$. The general proof follows similarly and is given below. 
    
    By letting $\pi_1, \pi_2$ be the determinisic policies which choose opposite actions at opposite contexts, i.e., $\pi_1(x_i) = a_i, \pi_2(x_i) = a_{1-i}$, we can choose $\Poffline(x) = d^*(\pi_1(x))$ and $\otherPoffline(x) = d^*(\pi_2(x))$ which yield 
    \begin{align*}
    	    \doffline^{\pi_1}(a)
    	    &=
    	    \sum_{i=0}^1 \Poffline(x_i)\mathbf{1}\brk[c]*{a = \pi_1(x_i)} \\
    	    &=
    	    \sum_{i=1}^2 d^*(\pi_1(x_i))\mathbf{1}\brk[c]*{a = \pi_1(x_i)} \\
    	    &=
    	    \sum_{i=1}^k d^*(a_i)\mathbf{1}\brk[c]*{a_i = a} 
    	    :=
    	    d^*(a).
    \end{align*}
    Similarly, $d_{\otherPoffline}^{\pi_2}(a) = d^*(a)$.

    For the second part of the proof choose $r_1(a, x) = \mathbf{1}\brk[c]*{x=x_i, a=a_i}$ and $r_2(a, x) = \mathbf{1}\brk[c]*{x=x_i, a = a_{1-i}}$. Notice that $\pi_i$ is optimal for $r_i$ under any distribution of contexts, yet $\pi_i$ achieves zero reward for $r_{1-i}$. 
    
    We now provide a complete proof for the general case.
    
    Let $\Ponline, d^*(a)$. Without loss of generality, let $\X = \brk[c]*{x_0, \hdots, x_m}$, $\A = \brk[c]*{a_0, \hdots, a_k}$ with $m \geq k$, and denote ${\X_k = \brk[c]*{x_1, \hdots, x_k} \subseteq \X}$. By definition there exists an injective function from $\A$ into $\X$. 
    
    Define
    \begin{align*}
        f(x)
        &=
        \begin{cases}
            a_i &, x = x_i, i = 0, \hdots, k \\
            a_0 &, \text{o.w.}
        \end{cases} \\
        g(x)
        &=
        \begin{cases}
            a_{i+1\Mod{k}} &, x = x_i, i = 0, \hdots, k \\
            a_0 &, \text{o.w.}
        \end{cases}
    \end{align*}
    
	Then we can select $\pi_1, \pi_2, \Poffline, \otherPoffline$ as follows
	\begin{align*}
	    \pi_1(a | x) &= \mathbf{1}\brk[c]*{a = f(x), x \in \X_k} + \frac{1}{k+1}\mathbf{1}\brk[c]*{x \notin \X_k}\\
	    \pi_2(a | x) &= \mathbf{1}\brk[c]*{a = g(x), x \in \X_k} + \frac{1}{k+1}\mathbf{1}\brk[c]*{x \notin \X_k},
	\end{align*}
	and
	\begin{align*}
	    \Poffline(x) &= d^*(f(x))\mathbf{1}\brk[c]*{x \in \X_k}, \\
	    \otherPoffline(x) &= d^*(g(x))\mathbf{1}\brk[c]*{x \in \X_k}.
	\end{align*}
	We get that
	\begin{align*}
	    \doffline^{\pi_1}(a)
	    &=
	    \sum_{i=1}^m \Poffline(x_i)\pi_1(a|x_i) \\
	    &=
	    \sum_{i=1}^k d^*(f(x_i))\mathbf{1}\brk[c]*{a = f(x_i)} \\
	    &=
	    \sum_{i=1}^k d^*(a_i)\mathbf{1}\brk[c]*{a_i = a} 
	    =
	    d^*(a).
	\end{align*}
	Similarly,
	\begin{align*}
	    d_{\otherPoffline}^{\pi_2}(a)
	    &=
	    \sum_{i=1}^k d^*(g(x_i))\mathbf{1}\brk[c]*{a = g(x_i)} \\
	    &=
	    \sum_{i=1}^k d^*(a_{i+1\Mod{k}})\mathbf{1}\brk[c]*{a_{i+1\Mod{k}} = a} \\
	    &=
	    \sum_{i=1}^k d^*(a_i)\mathbf{1}\brk[c]*{a_i = a} 
	    =
	    d^*(a).
	\end{align*}
	This proves the first part of the theorem. For the other parts, choose $r_1, r_2$ as follows
	\begin{align*}
	    r_1(a, x) &= \mathbf{1}\brk[c]*{x=x_i, a=a_i, 0 \leq i \leq k} \\
	    r_2(a, x) &= \mathbf{1}\brk[c]*{x=x_i, a=a_{i+1\Mod{k}}, 0 \leq i \leq k}.
	\end{align*}
	Then, by definition, for any $P(x)$ such that $\text{Supp}(P) \cap \X_k \neq \emptyset$,
	\begin{align*}
	    \expect*{x \sim P(x), a \sim \pi_1(\cdot | x)}{r_1(a, x)}
	    =
	    1
	    =
	    \max_{\pi \in \Pi}
	    \expect*{x \sim P(x), a \sim \pi(\cdot | x)}{r_1(a, x)}, \\
	    \expect*{x \sim P(x), a \sim \pi_1(\cdot | x)}{r_2(a, x)}
	    =
	    0
	    =
	    \min_{\pi \in \Pi}
	    \expect*{x \sim P(x), a \sim \pi(\cdot | x)}{r_2(a, x)}.
	\end{align*}
	And similarly,
	\begin{align*}
	    \expect*{x \sim P(x), a \sim \pi_2(\cdot | x)}{r_1(a, x)}
	    =
	    0
	    =
	    \min_{\pi \in \Pi}
	    \expect*{x \sim P(x), a \sim \pi(\cdot | x)}{r_1(a, x)}, \\
	    \expect*{x \sim P(x), a \sim \pi_2(\cdot | x)}{r_2(a, x)}
	    =
	    1
	    =
	    \max_{\pi \in \Pi}
	    \expect*{x \sim P(x), a \sim \pi(\cdot | x)}{r_2(a, x)}.
	\end{align*}
	The condition on the support holds for $\Poffline, \otherPoffline$ by definition. If, $\text{Supp}(\Ponline) \cap \X_k = \emptyset$, then the result holds trivially as $\expect*{x \sim \Ponline(x), a \sim \pi(\cdot | x)}{r_1(a, x)} = \expect*{x \sim \Ponline(x), a \sim \pi(\cdot | x)}{r_2(a, x)} = 0$ for all $\pi \in \Pi$. This completes the proof.
\end{proof}

\begin{lemma}
\label{lemma: optimal policy for every x}
Assume $\text{Supp}(\Ponline) \subseteq \text{Supp}(\Poffline)$. Then
    \begin{align*}
        \arg\max_\pi \expect*{x \sim \Poffline(x), s,a \sim d^{\pi}(s,a|x)}{r(s,a,x)}
        \subseteq
        \arg\max_\pi \expect*{x \sim \Ponline(x), s,a \sim d^{\pi}(s,a|x)}{r(s,a,x)}
    \end{align*}
\end{lemma}
\begin{proof}
    For clarity we denote
    \begin{align*}
        &\Pi^*_{\Poffline} = \arg\max_\pi \expect*{x \sim \Poffline(x), s,a \sim d^{\pi}(s,a|x)}{r(s,a,x)} \\
        &\Pi^*_{\Ponline} = \arg\max_\pi \expect*{x \sim \Ponline(x), s,a \sim d^{\pi}(s,a|x)}{r(s,a,x)} \\
        &\Pi^*_{\text{Supp}(\Poffline)} = \bigtimes_{x \in \text{Supp}(\Poffline)} \arg\max_\pi \expect*{s,a \sim d^{\pi}(s,a|x)}{r(s,a,x)}.
    \end{align*}
    To prove the lemma, we will show $\Pi^*_{\Poffline} = \Pi^*_{\text{Supp}(\Poffline)} \subseteq \Pi^*_{\Ponline}$.
    
    We begin by proving $\Pi^*_{\Poffline} = \Pi^*_{\text{Supp}(\Poffline)}$. Indeed, let $\pi^* \in \Pi^*_{\text{Supp}(\Poffline)}$. Then, for any $x \in \text{Supp}(\Poffline)$
    \begin{align*}
        \expect*{s,a \sim d^{\pi^*}(s,a|x)}{r(s,a,x)}
        =
        \max_\pi \expect*{s,a \sim d^{\pi}(s,a|x)}{r(s,a,x)}.
    \end{align*}
    In particular,
    \begin{align*}
        \expect*{x \sim \Poffline(x), s,a \sim d^{\pi^*}(s,a|x)}{r(s,a,x)}
        =
        \expect*{x \sim \Poffline(x)}{\max_\pi \expect*{s,a \sim d^{\pi}(s,a|x)}{r(s,a,x)}}
        \geq
        \max_\pi \expect*{x \sim \Poffline(x), s,a \sim d^{\pi^*}(s,a|x)}{r(s,a,x)},
    \end{align*}
    where we used Jensen's inequality. This proves $\Pi^*_{\text{Supp}(\Poffline)} \subseteq \Pi^*_{\Poffline}$. 
    
    To see the other direction, let $\pi_e \in \Pi^*_{\Poffline}$ and assume by contradiction that $\pi_e \notin \Pi^*_{\text{Supp}(\Poffline)}$. Then, there exists $\tilde{x} \in \text{Supp}(\Poffline)$ such that
    \begin{align*}
        \expect*{s,a \sim d^{\pi_e}(s,a|\tilde{x})}{r(s,a,\tilde{x})}
        <
        \max_\pi \expect*{s,a \sim d^{\pi}(s,a|\tilde{x})}{r(s,a,\tilde{x})}.
    \end{align*}
    Define
    \begin{align*}
        \tilde{\pi}(\cdot | s, x)
        =
        \mathbf{1}\brk[c]*{x = \tilde{x}}\pi_{\tilde{x}}(\cdot | s, \tilde{x})
        +
        \mathbf{1}\brk[c]*{x \neq \tilde{x}}\pi_e(\cdot | s, x),
    \end{align*}
    where $\pi_{\tilde{x}} \in \arg\max_\pi \expect*{s,a \sim d^{\pi}(s,a|\tilde{x})}{r(s,a,\tilde{x})}.$ Then,
    \begin{align*}
        v(\pi_e)
        &=
        P(x = \tilde{x})
        \expect*{s,a \sim d^{\pi_e}(s,a|\tilde{x})}{r(s,a,\tilde{x})}
        +
        \sum_{x \in \text{Supp}(\Poffline) \backslash \{\tilde{x}\}}
        P(x)
        \expect*{s,a \sim d^{\pi_e}(s,a|x)}{r(s,a,x)} \\
        &<
        P(x = \tilde{x})
        \expect*{s,a \sim d^{\tilde{\pi}}(s,a|\tilde{x})}{r(s,a,\tilde{x})}
        +
        \sum_{x \in \text{Supp}(\Poffline) \backslash \{\tilde{x}\}}
        P(x)
        \expect*{s,a \sim d^{\pi_e}(s,a|x)}{r(s,a,x)} 
        =
        v(\tilde{\pi}),
    \end{align*}
    in contradiction to $\pi_e \in \Pi^*_{\Poffline}$. This proves $\Pi^*_{\Poffline} \subseteq \Pi^*_{\text{Supp}(\Poffline)}$. We have thus shown that $\Pi^*_{\Poffline} = \Pi^*_{\text{Supp}(\Poffline)}$.
    
    Finally, it is left to show that $\Pi^*_{\text{Supp}(\Poffline)} \subseteq \Pi^*_{\Ponline}$. Similar to before, let $\pi^* \in \Pi^*_{\text{Supp}(\Poffline)}$. Then, for any $x \in \text{Supp}(\Poffline)$, by Jensen's inequality
    \begin{align*}
        \expect*{x \sim \Ponline(x), s,a \sim d^{\pi^*}(s,a|x)}{r(s,a,x)}
        =
        \expect*{x \sim \Ponline(x)}{\max_\pi \expect*{s,a \sim d^{\pi}(s,a|x)}{r(s,a,x)}}
        \geq
        \max_\pi \expect*{x \sim \Ponline(x), s,a \sim d^{\pi^*}(s,a|x)}{r(s,a,x)}.
    \end{align*}
    This completes the proof.
\end{proof}

\contextfreereward*
\begin{proof}
    Let $\pi_0 \in \equivset{\pi^*}$, we will show $\pi_0 \in \Pi^*_\M$. Since $r(s,a,x) = r(s,a,x')$ for all $x \in \X$ we denote $r(s,a) = r(s,a,x)$. By definition of $\equivset{\pi^*}$ we have that.
    \begin{align*}
        \donline^{\pi_0}(s,a) = \doffline^{\pi^*}(s,a)
    \end{align*}
    Then,
    \begin{align*}
        v(\pi_0)
        &=
        \expect*{x \sim \Ponline(x), s,a \sim d^{\pi_0}(s,a|x)}{r(s,a)} \\
        &=
        \expect*{x \sim \Ponline(x)}{\sum_{s \in \s,a \in \A}  d^{\pi_0}(s,a\mid x) r(s,a)} \\
        &=
        \sum_{s \in \s,a \in \A} r(s,a) \expect*{x \sim \Ponline(x)}{  d^{\pi_0}(s,a\mid x) } \\
        &=
        \expect*{s,a \sim \donline^{\pi_0}(s,a)}{r(s,a)} \\
        &=
        \expect*{s,a \sim \doffline^{\pi^*}(s,a)}{r(s,a)} \\
        &=
        \expect*{x \sim \Poffline(x), s,a \sim d^{\pi^*}(s,a|x)}{r(s,a)} \\
        &=
        \max_\pi \expect*{x \sim \Poffline(x), s,a \sim d^{\pi}(s,a|x)}{r(s,a)}
    \end{align*}
    Then, $\pi_0 \in \arg\max_\pi \expect*{x \sim \Poffline(x), s,a \sim d^{\pi}(s,a|x)}{r(s,a)}$. Applying \Cref{lemma: optimal policy for every x}
    \begin{align*}
        \pi_0 
        \in 
        \arg\max_{\pi}
        \expect*{x \sim \Ponline(x), s,a \sim d^{\pi}(s,a|x)}{r(s,a)}
        =
        \Pi^*_\M,
    \end{align*}
    completing the proof.
\end{proof}

\subsection{Proofs for \Cref{section: rl}}

\samplingprop*
\begin{proof}
We can write
    \begin{align*}
    d^{\pi}(s,a \mid x)
    &=
    (1-\gamma) \sum_{t=0}^\infty \gamma^t P(s_t = s, a_t = a | x) \\
    &=
    (1-\gamma) \sum_\tau \sum_{t=0}^\infty \gamma^t P(s_t = s, a_t = a | x, \tau)P(\tau | x) \\
    &=
    (1-\gamma) \sum_\tau \sum_{t=0}^\infty \gamma^t \mathbf{1}\brk[c]*{\tau_t = (s, a)}P(\tau | x).
\end{align*}
Then, denoting $P^{\pi}_{\rho_s^*}(\tau) = \expect*{x \sim \rho_s^*}{P(\tau | x)}$, we get that
\begin{align*}
    d^{\pi}_{\rho_s^*}(s,a) 
    &=
    (1-\gamma) \sum_\tau \sum_{t=0}^\infty \gamma^t \mathbf{1}\brk[c]*{\tau_t = (s, a)}P^{\pi}_{\rho_s^*}(\tau) \\
    &=
    \expect*{\tau \sim P^\pi_{\rho_s^*}}{(1-\gamma) \sum_{t=0}^\infty \gamma^t \mathbf{1}\brk[c]*{\tau_t = (s, a)}}.
\end{align*}
Since, $\text{Supp}(\Ponline) \subseteq \text{Supp}(\Poffline)$, there exists $p^n \in \Delta_n$ such that 
$\expect*{i \sim p^n}{(1-\gamma) \sum_{t=0}^\infty \gamma^t \mathbf{1}\brk[c]*{s_t^i, a_t^i = (s, a)}}$ is an unbiased estimator of $d^{\pi}_{\rho_s^*}(s,a)$. The result follows by the law of large numbers.
\end{proof}

\samplecomplexity*
\begin{proof}
    We begin by showing that $h(P) = \min_{x \in \Delta_n} D_f(P || \expect*{x}{Q_x})$ is convex in $P$. We can write $D_f$ in its variational form, rewriting $h(P)$ as
\begin{align*}
    h(P) = \min_{x \in \Delta_n} \max_{g: \Z \mapsto \R} \expect*{z \sim P}{g(z)} - \expect*{x, z \sim Q_x}{f^*(g(z))},
\end{align*}
where
\begin{align*}
    f^*(w) = \sup_{y} \brk[c]*{yw - f(y)}.
\end{align*}
We have that $\expect*{z \sim P}{g(z)} - \expect*{x, z \sim Q_x}{f^*(g(z))}$ is affine in $g$ and $x$. Therefore, strong duality holds, yielding
\begin{align*}
    h(P) 
    &= 
    \max_{g: \Z \mapsto \R} \min_{x \in \Delta_n} \expect*{z \sim P}{g(z)} - \expect*{x, z \sim Q_x}{f^*(g(z))} \\
    &=
    \max_{g: \Z \mapsto \R} \brk[c]*{\expect*{z \sim P}{g(z)} + \brk*{\max_{x \in \Delta_n} \expect*{x, z \sim Q_x}{f^*(g(z))}}}
\end{align*}
We have that $\max_{x \in \Delta_n} \expect*{x, z \sim Q_x}{f^*(g(z))}$ is convex in $g$ as a maximum over convex (affine) functions in a compact set. Therefore $h(P)$ is also convex as a maximum over convex functions.

Then, the objective in $Problem~(\ref{eq: max min min problem})$ is convex in $\donline^{\pi}$. Following the meta algorithm framework for convex RL in \citet{zahavy2021reward}, we write the gradient of $D_f(\donline^{\pi}(s,a) || \doffline^{\pi^*}(s,a))$. Notice that for any general $f$-divergence 
$
    D_f(x_i || y_i) 
    = 
    \expect*{y_i}{f\brk*{\frac{x_i}{y_i}}}
$
it holds that
\begin{align*}
    \nabla_{x_j} D_f(x_i || y_i) = 0, j \neq i,
\end{align*}
and
\begin{align*}
    \nabla_{x_i} D_f(x_i || y_i) 
    = 
    \nabla_{x_i} \expect*{y_i}{f\brk*{\frac{x_i}{y_i}}} 
    =
    \expect*{y_i}{\frac{1}{y_i}\nabla_z f\brk*{z}\mid_{z=\frac{x_i}{y_i}}}.
\end{align*}
Specifically, for the $KL$-divergence,
$
    D_{KL}(p_i || q_i) 
    = 
    -\expect*{q_i}{\log\brk*{\frac{p_i}{q_i}}}.
$
Then,
\begin{align*}
    \nabla_{p_i} D_{KL}(p_i || q_i) 
    =
    \expect*{q_i}{\frac{1}{p_i}}.
\end{align*}
Applying Lemma~2 of \citet{zahavy2021reward} with a Follow the Leader (FTL) cost player completes the proof.
\end{proof}

\subsection{Proofs of Additional Results in Appendix}

\algooneconvergence*
\begin{proof}
    Denote 
    \begin{align*}
    \lambda^*_1 
    &= 
    \max_{\pi \in \Pi_{\text{det}}, \pi \not \in \equivset{\pi^*}, \pi' \in \equivset{\pi^*}} d_{TV}\brk*{\donline^{\pi}(s,a,x), \donline^{\pi'}(s,a,x)}, \\ 
    \lambda^*_2 
    &= 
    \min_{\pi \in \Pi_{\text{det}}, \pi \not \in \equivset{\pi^*}} d_{TV}(\donline^{\pi}(s,a), \doffline^{\pi^*}(s,a)),
    \end{align*} 
    where $d_{TV}$ is the total variation distance. Let $\lambda^* = \frac{\lambda^*_2}{\lambda^*_1}$ and $\lambda \in (0, \lambda^*)$ and notice that $\lambda^* > 0$.
    % since $f$ is not affine and thus $D_f(P || Q) = 0$ if and only if $P = Q$ (see e.g., Proposition~1 in \citet{sason2018f}).
    
    To prove the result., we will show that at iteration $n$ of the algorithm $\pi_n \in \equivset{\pi^*}$ and that either $\pi_n \notin \Upsilon_{n-1} := \brk[c]*{\pi_j}_{j=1}^{n-1}$ or $\Upsilon_{n-1} = \equivset{\pi^*}$.
    
    \textbf{Base case ($n = 1$).}  By the variational representation of the $f$-divergence,
    \begin{align*}
        \max_{g_0 : \s \times \A \mapsto \R}
        \expect*{s, a \sim \donline^{\pi}(s, a)}{
         g_0(s,a)}
        -
        \expect*{s, a \sim \doffline^{\pi^*}(s, a)}{ f^*(g_0(s,a))} 
        =
        d_{TV}(\donline^{\pi}(s,a), \doffline^{\pi^*}(s,a)).
    \end{align*}
    By definition  $\equivset{\pi^*} = \arg\min_{\pi \in \Pi_{\text{det}}} d_{TV}(\donline^{\pi}(s,a) || \doffline^{\pi^*}(s,a))$. Then, $\pi_1 \in \equivset{\pi^*}$. Finally since $\Upsilon_0 = \emptyset$, we have that $\pi_1 \notin \Upsilon_0$. 
    
    \textbf{Induction step.} Suppose the claim holds for some $n = k$. We will show it holds for $n = k+1$.
    
    We begin by showing that $\pi_{k+1} \in \equivset{\pi^*}$. Assume by contradiction that $\pi_{k+1} \in \Pi_{\text{det}}, \pi_{k+1} \notin \equivset{\pi^*}$. Using the variational form of the $f$-divergence,
    \begin{align*}
        \max_{g_i: \s \times \A \times \X} L_i(\pi_{k+1}; g_i)
        &=
        d_{TV}(\donline^{\pi_{k+1}}(s,a,x), \donline^{\pi_i}(s,a,x)) 
        \leq 
        \lambda^*_1, \\
        \max_{g_0: \s \times \A} L^*(\pi_{k+1}; g_0)
        &=
        d_{TV}(\donline^{\pi_{k+1}}(s,a), \doffline^{\pi^*}(s,a)) 
        \geq 
        \lambda^*_2.
    \end{align*}
    We have that
    \begin{align*}
        \max_{\substack{g_0: \s \times \A \mapsto \R, \\ g_i: \s \times \A \times \X \mapsto \R}}
        L^*(\pi_{k+1}; g_0) 
        - 
        \lambda
        \min_i
        L_i(\pi_{k+1}; g_i)
        \geq
        \lambda^*_2
        - 
        \lambda
        \lambda^*_1 
        >
        \lambda^*_2
        - 
        \lambda^*
        \lambda^*_1
        =
        0.
    \end{align*}
    Next, let $\tilde{\pi}_{k+1} \in \equivset{\pi^*}$, then
    
    \begin{align*}
        \max_{\substack{g_0: \s \times \A \mapsto \R, \\ g_i: \s \times \A \times \X \mapsto \R}}
        L^*(\tilde{\pi}_{k+1}; g_0) 
        - 
        \lambda
        \min_i
        L_i(\tilde{\pi}_{k+1}; g_i)
        \leq
        0,
    \end{align*}
    where we used the fact that $L^*(\tilde{\pi}_{k+1}; g_0) = 0$ by definition of $\equivset{\pi^*}$, and $L_i \geq 0$. We have reached a contradiction to $\pi_{k+1}$ being a solution to \Cref{eq: confounded imitation}. This proves that $\pi_{k+1} \in \equivset{\pi^*}$.
    
    Finally, we show that $\pi_{k+1} \notin \Upsilon_k$ if and only if $\Upsilon_{k} \neq \equivset{\pi^*}$. First, notice that if $\Upsilon_{k} = \equivset{\pi^*}$ then \Cref{eq: confounded imitation} will return $\pi_{k+1} \in \Upsilon_k$ by definition of the total variation distance. Next, assume $\Upsilon_{k} \neq \equivset{\pi^*}$ and assume by contradiction ${\pi_{k+1} \in \Upsilon_k}$. Then, ${\exists i: \max_{g_i} L_i(\pi_{k+1}; g_i) = 0}$, and $\max_{g_0: \s \times \A \mapsto \R} L^*(\pi_{k+1}; g_0) = 0$, by definition of $\equivset{\pi^*}$. Hence,
    \begin{align*}
        \max_{\substack{g_0: \s \times \A \mapsto \R, \\ g_i: \s \times \A \times \X \mapsto \R}}
        L^*(\pi_{k+1}; g_0) 
        - 
        \lambda
        \min_i
        L_i(\pi_{k+1}; g_i)
        =
        0.
    \end{align*}
    In contrast, since $\Upsilon_{k} \neq \equivset{\pi^*}$, there exists $\tilde{\pi} \in \equivset{\pi^*}$ such that $\tilde{\pi} \notin \Upsilon_{k}$, and
    \begin{align*}
        \max_{\substack{g_0: \s \times \A \mapsto \R, \\ g_i: \s \times \A \times \X \mapsto \R}}
        L^*(\tilde{\pi}; g_0) 
        - 
        \lambda
        \min_i
        L_i(\tilde{\pi}; g_i)
        \leq
        \lambda_1^*
        < 0,
    \end{align*}
    in contradiction to $\pi_{k+1}$ being a solution \Cref{eq: confounded imitation}. This completes the proof.
\end{proof}

\deltasufficiency*
\begin{proof}
    Let $\pi \in \Pi$. We will show that $\pi \in \equivset{\pi}^{\Gamma - 1}$. By elementary algebra, we have that, under \Cref{assumption: sensitivity},
    \begin{align*}
        \Ponline(x)(1 - \Gamma^{-1}) + \Gamma^{-1}
        \leq
        \frac{\Ponline(x)}{\Poffline(x)} 
        \leq
        \Ponline(x)(1 - \Gamma) + \Gamma.
    \end{align*}
    Since $\text{Supp}\brk*{\Poffline} \subseteq \text{Supp}\brk*{\Ponline}$,
    \begin{align*}
        \donline^\pi(s,a)
        &=
        \expect*{x \sim \Ponline(x)}{d^{\pi}(s,a \mid x)} \\
        &=
        \expect*{x \sim \Poffline(x)}{\frac{\Ponline(x)}{\Poffline(x)}d^{\pi}(s,a\mid x)} \\
        &\leq
        \expect*{x \sim \Poffline(x)}{\brk*{\Ponline(x)(1 - \Gamma) + \Gamma}d^{\pi}(s,a\mid x)}.
    \end{align*}
    Subtracting $\doffline^\pi$ from both sides we get that
    \begin{align*}
        \donline^\pi(s,a) - \doffline^\pi(s,a)
        &\leq
        \expect*{x \sim \Poffline(x)}{\brk*{\Ponline(x)(1 - \Gamma) + \Gamma - 1}d^{\pi}(s,a\mid x)} \\
        &=
        \brk*{\Gamma - 1}
        \expect*{x \sim \Poffline(x)}{\brk*{1 - \Ponline(x)}d^{\pi}(s,a\mid x)} \\
        &\leq
        \Gamma - 1.
    \end{align*}
    Similarly,
    \begin{align*}
        \donline^\pi
        \geq
        \expect*{x \sim \Poffline(x)}{\brk*{\Ponline(x)(1 - \Gamma^{-1}) + \Gamma^{-1}}d^{\pi}(s,a\mid x)}.
    \end{align*}
    Hence,
    \begin{align*}
        \donline^\pi(s,a) - \doffline^\pi(s,a)
        &\geq
        \expect*{x \sim \Poffline(x)}{\brk*{\Ponline(x)(1 - \Gamma^{-1}) + \Gamma^{-1} - 1}d^{\pi}(s,a\mid x)} \\
        &=
        \brk*{\Gamma^{-1} - 1}
        \expect*{x \sim \Poffline(x)}{\brk*{1 - \Ponline(x)}d^{\pi}(s,a\mid x)} \\
        &\geq
        -(1- \Gamma^{-1}) \\
        &\geq 
        -(\Gamma - 1)
    \end{align*}
    where the last two transitions hold since $\Gamma \geq 1$. 
    Then, we have that
    \begin{align*}
        \abs{\donline^\pi(s,a) - \doffline^\pi(s,a)}
        \leq
        \Gamma - 1.
    \end{align*}
    This completes the proof.
\end{proof}

\contextdependentreward*
\begin{proof}
    Let $\pi^* \in \Pi^*_\M$, $\pi_0 \in \equivset{\pi^*}$. The proof follows similar steps to that of \Cref{thm: context free reward}.
    \begin{align*}
        v(\pi_0)
        &=
        \expect*{x \sim \Ponline(x), s,a \sim d^{\pi_0}(s,a|x)}{r(s,a,x)} \\
        &=
        \expect*{x \sim \Ponline(x), s,a \sim d^{\pi_0}(s,a|x)}{r(s,a,x)-r_0(s,a)} 
        +
        \expect*{x \sim \Ponline(x), s,a \sim d^{\pi_0}(s,a|x)}{r_0(s,a)}\\
        &=
        \expect*{x \sim \Ponline(x), s,a \sim d^{\pi_0}(s,a|x)}{r(s,a,x)-r_0(s,a)} 
        +
        \expect*{x \sim \Ponline(x)}{\sum_{s \in \s,a \in \A}  d^{\pi_0}(s,a\mid x) r_0(s,a)} \\
        &=
        \expect*{x \sim \Ponline(x), s,a \sim d^{\pi_0}(s,a|x)}{r(s,a,x)-r_0(s,a)} 
        +
        \sum_{s \in \s,a \in \A} r_0(s,a) \expect*{x \sim \Ponline(x)}{  d^{\pi_0}(s,a\mid x) } \\
        &=
        \expect*{x \sim \Ponline(x), s,a \sim d^{\pi_0}(s,a|x)}{r(s,a,x)-r_0(s,a)} 
        +
        \expect*{s,a \sim \donline^{\pi_0}(s,a)}{r_0(s,a)} \\
        &=
        \expect*{x \sim \Ponline(x), s,a \sim d^{\pi_0}(s,a|x)}{r(s,a,x)-r_0(s,a)} 
        +
        \expect*{s,a \sim \doffline^{\pi^*}(s,a)}{r_0(s,a)} \\
        &=
        \expect*{x \sim \Ponline(x), s,a \sim d^{\pi_0}(s,a|x)}{r(s,a,x)-r_0(s,a)} 
        +
        \expect*{x \sim \Poffline(x), s,a \sim d^{\pi^*}(s,a|x)}{r_0(s,a)}.
    \end{align*}
    
    Then,
    \begin{align*}
        v(\pi^*) - v(\pi_0)
        &=
        \expect*{x \sim \Poffline(x), s,a \sim d^{\pi^*}(s,a|x)}{r(s,a,x) - r_0(s,a)}
        -
        \expect*{x \sim \Ponline(x), s,a \sim d^{\pi_0}(s,a|x)}{r(s,a,x) - r_0(s,a)} \\
        &\leq
        \expect*{x \sim \Poffline(x)}{\epsilon(x)}
        +
        \expect*{x \sim \Ponline(x)}{\epsilon(x)},
    \end{align*}
    completing the proof.
\end{proof}

\end{document}

%% file: math_commands.tex
\usepackage{amsmath,amsfonts,bm}

\def\eqref#1{equation~\ref{#1}}
\def\1{\bm{1}}

\DeclareMathAlphabet{\mathsfit}{\encodingdefault}{\sfdefault}{m}{sl}
\SetMathAlphabet{\mathsfit}{bold}{\encodingdefault}{\sfdefault}{bx}{n}

\newcommand{\R}{\mathbb{R}}

%% file: main.bbl
\begin{thebibliography}{46}
\providecommand{\natexlab}[1]{#1}
\providecommand{\url}[1]{\texttt{#1}}
\expandafter\ifx\csname urlstyle\endcsname\relax
  \providecommand{\doi}[1]{doi: #1}\else
  \providecommand{\doi}{doi: \begingroup \urlstyle{rm}\Url}\fi

\bibitem[Hussein et~al.(2017)Hussein, Gaber, Elyan, and
  Jayne]{hussein2017imitation}
Ahmed Hussein, Mohamed~Medhat Gaber, Eyad Elyan, and Chrisina Jayne.
\newblock Imitation learning: A survey of learning methods.
\newblock \emph{ACM Computing Surveys (CSUR)}, 50\penalty0 (2):\penalty0 1--35,
  2017.

\bibitem[Gottesman et~al.(2019)Gottesman, Johansson, Komorowski, Faisal,
  Sontag, Doshi-Velez, and Celi]{gottesman2019guidelines}
Omer Gottesman, Fredrik Johansson, Matthieu Komorowski, Aldo Faisal, David
  Sontag, Finale Doshi-Velez, and Leo~Anthony Celi.
\newblock Guidelines for reinforcement learning in healthcare.
\newblock \emph{Nature medicine}, 25\penalty0 (1):\penalty0 16--18, 2019.

\bibitem[Hallak et~al.(2015)Hallak, Di~Castro, and
  Mannor]{hallak2015contextual}
Assaf Hallak, Dotan Di~Castro, and Shie Mannor.
\newblock Contextual markov decision processes.
\newblock \emph{arXiv preprint arXiv:1502.02259}, 2015.

\bibitem[Ie et~al.(2019)Ie, Hsu, Mladenov, Jain, Narvekar, Wang, Wu, and
  Boutilier]{ie2019recsim}
Eugene Ie, Chih-wei Hsu, Martin Mladenov, Vihan Jain, Sanmit Narvekar, Jing
  Wang, Rui Wu, and Craig Boutilier.
\newblock Recsim: A configurable simulation platform for recommender systems.
\newblock \emph{arXiv preprint arXiv:1909.04847}, 2019.

\bibitem[Erickson et~al.(2020)Erickson, Gangaram, Kapusta, Liu, and
  Kemp]{erickson2020assistive}
Zackory Erickson, Vamsee Gangaram, Ariel Kapusta, C~Karen Liu, and Charles~C
  Kemp.
\newblock Assistive gym: A physics simulation framework for assistive robotics.
\newblock In \emph{2020 IEEE International Conference on Robotics and
  Automation (ICRA)}, pages 10169--10176. IEEE, 2020.

\bibitem[Ho and Ermon(2016)]{ho2016generative}
Jonathan Ho and Stefano Ermon.
\newblock Generative adversarial imitation learning.
\newblock \emph{Advances in neural information processing systems},
  29:\penalty0 4565--4573, 2016.

\bibitem[Fu et~al.(2017)Fu, Luo, and Levine]{fu2017learning}
Justin Fu, Katie Luo, and Sergey Levine.
\newblock Learning robust rewards with adversarial inverse reinforcement
  learning.
\newblock \emph{arXiv preprint arXiv:1710.11248}, 2017.

\bibitem[Kostrikov et~al.(2019)Kostrikov, Nachum, and
  Tompson]{kostrikov2019imitation}
Ilya Kostrikov, Ofir Nachum, and Jonathan Tompson.
\newblock Imitation learning via off-policy distribution matching.
\newblock \emph{arXiv preprint arXiv:1912.05032}, 2019.

\bibitem[Brantley et~al.(2019)Brantley, Sun, and
  Henaff]{brantley2019disagreement}
Kiant{\'e} Brantley, Wen Sun, and Mikael Henaff.
\newblock Disagreement-regularized imitation learning.
\newblock In \emph{International Conference on Learning Representations}, 2019.

\bibitem[Pearl(2009{\natexlab{a}})]{pearl2009causal}
Judea Pearl.
\newblock Causal inference in statistics: An overview.
\newblock \emph{Statistics surveys}, 3:\penalty0 96--146, 2009{\natexlab{a}}.

\bibitem[Nachum et~al.(2019)Nachum, Dai, Kostrikov, Chow, Li, and
  Schuurmans]{nachum2019algaedice}
Ofir Nachum, Bo~Dai, Ilya Kostrikov, Yinlam Chow, Lihong Li, and Dale
  Schuurmans.
\newblock Algaedice: Policy gradient from arbitrary experience.
\newblock \emph{arXiv preprint arXiv:1912.02074}, 2019.

\bibitem[Zahavy et~al.(2021)Zahavy, O'Donoghue, Desjardins, and
  Singh]{zahavy2021reward}
Tom Zahavy, Brendan O'Donoghue, Guillaume Desjardins, and Satinder Singh.
\newblock Reward is enough for convex mdps.
\newblock \emph{arXiv preprint arXiv:2106.00661}, 2021.

\bibitem[Schulman et~al.(2017)Schulman, Wolski, Dhariwal, Radford, and
  Klimov]{schulman2017proximal}
John Schulman, Filip Wolski, Prafulla Dhariwal, Alec Radford, and Oleg Klimov.
\newblock Proximal policy optimization algorithms.
\newblock \emph{arXiv preprint arXiv:1707.06347}, 2017.

\bibitem[Liang et~al.(2018)Liang, Liaw, Nishihara, Moritz, Fox, Goldberg,
  Gonzalez, Jordan, and Stoica]{liang2018rllib}
Eric Liang, Richard Liaw, Robert Nishihara, Philipp Moritz, Roy Fox, Ken
  Goldberg, Joseph Gonzalez, Michael Jordan, and Ion Stoica.
\newblock Rllib: Abstractions for distributed reinforcement learning.
\newblock In \emph{International Conference on Machine Learning}, pages
  3053--3062. PMLR, 2018.

\bibitem[Pomerleau(1989)]{pomerleau1989alvinn}
Dean~A Pomerleau.
\newblock Alvinn: An autonomous land vehicle in a neural network.
\newblock Technical report, CARNEGIE-MELLON UNIV PITTSBURGH PA ARTIFICIAL
  INTELLIGENCE AND PSYCHOLOGY~…, 1989.

\bibitem[Bratko et~al.(1995)Bratko, Urban{\v{c}}i{\v{c}}, and
  Sammut]{bratko1995behavioural}
Ivan Bratko, Tanja Urban{\v{c}}i{\v{c}}, and Claude Sammut.
\newblock Behavioural cloning: phenomena, results and problems.
\newblock \emph{IFAC Proceedings Volumes}, 28\penalty0 (21):\penalty0 143--149,
  1995.

\bibitem[Fu et~al.(2018)Fu, Luo, and Levine]{fu2018learning}
Justin Fu, Katie Luo, and Sergey Levine.
\newblock Learning robust rewards with adverserial inverse reinforcement
  learning.
\newblock In \emph{International Conference on Learning Representations}, 2018.

\bibitem[Kim and Park(2018)]{kim2018imitation}
Kee-Eung Kim and Hyun~Soo Park.
\newblock Imitation learning via kernel mean embedding.
\newblock In \emph{Thirty-Second AAAI Conference on Artificial Intelligence},
  2018.

\bibitem[Levine et~al.(2020)Levine, Kumar, Tucker, and Fu]{levine2020offline}
Sergey Levine, Aviral Kumar, George Tucker, and Justin Fu.
\newblock Offline reinforcement learning: Tutorial, review, and perspectives on
  open problems.
\newblock \emph{arXiv preprint arXiv:2005.01643}, 2020.

\bibitem[Kumar et~al.(2020)Kumar, Zhou, Tucker, and
  Levine]{kumar2020conservative}
Aviral Kumar, Aurick Zhou, George Tucker, and Sergey Levine.
\newblock Conservative q-learning for offline reinforcement learning.
\newblock \emph{arXiv preprint arXiv:2006.04779}, 2020.

\bibitem[Kostrikov et~al.(2021)Kostrikov, Fergus, Tompson, and
  Nachum]{kostrikov2021offline}
Ilya Kostrikov, Rob Fergus, Jonathan Tompson, and Ofir Nachum.
\newblock Offline reinforcement learning with fisher divergence critic
  regularization.
\newblock In \emph{International Conference on Machine Learning}, pages
  5774--5783. PMLR, 2021.

\bibitem[Tennenholtz et~al.(2021{\natexlab{a}})Tennenholtz, Baram, and
  Mannor]{tennenholtz2021gelato}
Guy Tennenholtz, Nir Baram, and Shie Mannor.
\newblock Gelato: Geometrically enriched latent model for offline reinforcement
  learning.
\newblock \emph{arXiv preprint arXiv:2102.11327}, 2021{\natexlab{a}}.

\bibitem[Fujimoto and Gu(2021)]{fujimoto2021minimalist}
Scott Fujimoto and Shixiang~Shane Gu.
\newblock A minimalist approach to offline reinforcement learning.
\newblock \emph{arXiv preprint arXiv:2106.06860}, 2021.

\bibitem[Nair et~al.(2020)Nair, Dalal, Gupta, and Levine]{nair2020accelerating}
Ashvin Nair, Murtaza Dalal, Abhishek Gupta, and Sergey Levine.
\newblock Accelerating online reinforcement learning with offline datasets.
\newblock \emph{arXiv preprint arXiv:2006.09359}, 2020.

\bibitem[Peng et~al.(2019)Peng, Kumar, Zhang, and Levine]{peng2019advantage}
Xue~Bin Peng, Aviral Kumar, Grace Zhang, and Sergey Levine.
\newblock Advantage-weighted regression: Simple and scalable off-policy
  reinforcement learning.
\newblock \emph{arXiv preprint arXiv:1910.00177}, 2019.

\bibitem[Siegel et~al.(2019)Siegel, Springenberg, Berkenkamp, Abdolmaleki,
  Neunert, Lampe, Hafner, Heess, and Riedmiller]{siegel2019keep}
Noah Siegel, Jost~Tobias Springenberg, Felix Berkenkamp, Abbas Abdolmaleki,
  Michael Neunert, Thomas Lampe, Roland Hafner, Nicolas Heess, and Martin
  Riedmiller.
\newblock Keep doing what worked: Behavior modelling priors for offline
  reinforcement learning.
\newblock In \emph{International Conference on Learning Representations}, 2019.

\bibitem[Zhang et~al.(2020)Zhang, Kumor, and Bareinboim]{zhang2020causal}
Junzhe Zhang, Daniel Kumor, and Elias Bareinboim.
\newblock Causal imitation learning with unobserved confounders.
\newblock \emph{Advances in neural information processing systems}, 33, 2020.

\bibitem[de~Haan et~al.(2019)de~Haan, Jayaraman, and Levine]{de2019causal}
Pim de~Haan, Dinesh Jayaraman, and Sergey Levine.
\newblock Causal confusion in imitation learning.
\newblock \emph{Advances in Neural Information Processing Systems},
  32:\penalty0 11698--11709, 2019.

\bibitem[Lattimore et~al.(2016)Lattimore, Lattimore, and
  Reid]{lattimore2016causal}
Finnian Lattimore, Tor Lattimore, and Mark~D Reid.
\newblock Causal bandits: learning good interventions via causal inference.
\newblock In \emph{Proceedings of the 30th International Conference on Neural
  Information Processing Systems}, pages 1189--1197, 2016.

\bibitem[Tennenholtz et~al.(2021{\natexlab{b}})Tennenholtz, Shalit, Mannor, and
  Efroni]{tennenholtzbandits}
Guy Tennenholtz, Uri Shalit, Shie Mannor, and Yonathan Efroni.
\newblock Bandits with partially observable confounded data.
\newblock In \emph{Conference on Uncertainty in Artificial Intelligence}. PMLR,
  2021{\natexlab{b}}.

\bibitem[Zhang and Bareinboim(2019)]{zhang2019near}
Junzhe Zhang and Elias Bareinboim.
\newblock Near-optimal reinforcement learning in dynamic treatment regimes.
\newblock \emph{Advances in Neural Information Processing Systems},
  32:\penalty0 13401--13411, 2019.

\bibitem[Wang et~al.(2020)Wang, Yang, and Wang]{wang2020provably}
Lingxiao Wang, Zhuoran Yang, and Zhaoran Wang.
\newblock Provably efficient causal reinforcement learning with confounded
  observational data.
\newblock \emph{arXiv preprint arXiv:2006.12311}, 2020.

\bibitem[Tennenholtz et~al.(2020)Tennenholtz, Shalit, and
  Mannor]{tennenholtz2020off}
Guy Tennenholtz, Uri Shalit, and Shie Mannor.
\newblock Off-policy evaluation in partially observable environments.
\newblock In \emph{Proceedings of the AAAI Conference on Artificial
  Intelligence}, volume~34, pages 10276--10283, 2020.

\bibitem[Oberst and Sontag(2019)]{oberst2019counterfactual}
Michael Oberst and David Sontag.
\newblock Counterfactual off-policy evaluation with gumbel-max structural
  causal models.
\newblock In \emph{International Conference on Machine Learning}, pages
  4881--4890. PMLR, 2019.

\bibitem[Kallus and Zhou(2020)]{kallus2020confounding}
Nathan Kallus and Angela Zhou.
\newblock Confounding-robust policy evaluation in infinite-horizon
  reinforcement learning.
\newblock \emph{arXiv preprint arXiv:2002.04518}, 2020.

\bibitem[Csisz{\'a}r and Shields(2004)]{csiszar2004information}
Imre Csisz{\'a}r and Paul~C Shields.
\newblock Information theory and statistics: A tutorial.
\newblock 2004.

\bibitem[Liese and Vajda(2006)]{liese2006divergences}
Friedrich Liese and Igor Vajda.
\newblock On divergences and informations in statistics and information theory.
\newblock \emph{IEEE Transactions on Information Theory}, 52\penalty0
  (10):\penalty0 4394--4412, 2006.

\bibitem[Ke et~al.(2020)Ke, Choudhury, Barnes, Sun, Lee, and
  Srinivasa]{ke2020imitation}
Liyiming Ke, Sanjiban Choudhury, Matt Barnes, Wen Sun, Gilwoo Lee, and
  Siddhartha Srinivasa.
\newblock Imitation learning as f-divergence minimization.
\newblock In \emph{International Workshop on the Algorithmic Foundations of
  Robotics}, pages 313--329. Springer, 2020.

\bibitem[Puterman(2014)]{puterman2014markov}
Martin~L Puterman.
\newblock \emph{Markov decision processes: discrete stochastic dynamic
  programming}.
\newblock John Wiley \& Sons, 2014.

\bibitem[Hsu and Small(2013)]{hsu2013calibrating}
Jesse~Y Hsu and Dylan~S Small.
\newblock Calibrating sensitivity analyses to observed covariates in
  observational studies.
\newblock \emph{Biometrics}, 69\penalty0 (4):\penalty0 803--811, 2013.

\bibitem[Namkoong et~al.(2020)Namkoong, Keramati, Yadlowsky, and
  Brunskill]{namkoong2020off}
Hongseok Namkoong, Ramtin Keramati, Steve Yadlowsky, and Emma Brunskill.
\newblock Off-policy policy evaluation for sequential decisions under
  unobserved confounding.
\newblock \emph{arXiv preprint arXiv:2003.05623}, 2020.

\bibitem[Kallus and Zhou(2021)]{kallus2021minimax}
Nathan Kallus and Angela Zhou.
\newblock Minimax-optimal policy learning under unobserved confounding.
\newblock \emph{Management Science}, 67\penalty0 (5):\penalty0 2870--2890,
  2021.

\bibitem[Pearl(2009{\natexlab{b}})]{Pearl:2009:CMR:1642718}
Judea Pearl.
\newblock \emph{Causality: Models, Reasoning and Inference}.
\newblock Cambridge University Press, New York, NY, USA, 2nd edition,
  2009{\natexlab{b}}.
\newblock ISBN 052189560X, 9780521895606.

\bibitem[Correa and Bareinboim(2020)]{correa2020calculus}
Juan Correa and Elias Bareinboim.
\newblock A calculus for stochastic interventions: Causal effect identification
  and surrogate experiments.
\newblock In \emph{Proceedings of the AAAI Conference on Artificial
  Intelligence}, volume~34, pages 10093--10100, 2020.

\bibitem[Pearl(2000)]{pearl2000causality}
Judea Pearl.
\newblock \emph{Causality: models, reasoning and inference}, volume~29.
\newblock Cambridge University Press, 2000.

\bibitem[Rapin and Teytaud(2018)]{nevergrad}
J.~Rapin and O.~Teytaud.
\newblock {Nevergrad - A gradient-free optimization platform}.
\newblock \url{https://GitHub.com/FacebookResearch/Nevergrad}, 2018.

\end{thebibliography}
